\documentclass[10pt]{article}

\usepackage{amsmath,amsthm,amsfonts,mathrsfs,anysize,fancyhdr,epsfig,mdframed, float,graphicx,dsfont}

\usepackage{chngcntr}
\usepackage{algorithm,algorithmic}
\theoremstyle{plain}
\newtheorem{theorem}{Theorem}

\newtheorem{lemma}[theorem]{Lemma}                              

\usepackage{enumitem}
\theoremstyle{definition}

\newtheorem{example}[theorem]{Example}
\newtheorem{remark}[theorem]{Remark}

\usepackage{chngcntr}

\usepackage[paperwidth=8.5in,paperheight=11in,top=1.2in, bottom=1.2in, left=0.9in, right=0.9in]{geometry}
\usepackage{mathtools}                                                      
\mathtoolsset{showonlyrefs=true}                                    
\def \caratt {{\mathds{1}}}
\providecommand{\keywords}[1]{\textbf{Keywords:} #1}

\begin{document}
\title{Analytic expressions for the output evolution of a deep neural network}
\author{Anastasia Borovykh\thanks{CWI Amsterdam, the Netherlands. \textbf{e-mail}: anastasia.borovykh@cwi.nl}
}
\maketitle

\begin{abstract}
We present a novel methodology based on a Taylor expansion of the network output for obtaining analytical expressions for the expected value of the network weights and output under stochastic training. Using these analytical expressions the effects of the hyperparameters and the noise variance of the optimization algorithm on the performance of the deep neural network are studied. In the early phases of training with a small noise coefficient, the output is equivalent to a linear model. In this case the network can generalize better due to the noise preventing the output from fully converging on the train data, however the noise does not result in any explicit regularization. In the later training stages, when higher order approximations are required, the impact of the noise becomes more significant, i.e. in a model which is non-linear in the weights noise can regularize the output function resulting in better generalization as witnessed by its influence on the weight Hessian, a commonly used metric for generalization capabilities. 
\end{abstract}


\section{Introduction}
With the large number of applications which are nowadays in some way using deep learning, it is of significant value to gain insight into the output evolution of a deep neural network and the effects that the model architecture and optimization algorithm have on it.
A deep neural network is a complex model due to the non-linear dependencies and the large number of parameters in the model. Understanding the network output and its generalization capabilities, i.e. how well a model optimized on train data will be able to perform on unseen test data, is thus a complex task. One way of gaining insight into the network is by studying it in a large-parameter limit, a setting in which its dynamics becomes analytically tractable. Such limits have been considered in e.g. \cite{lee17}, \cite{matthews18}, \cite{jacot18}, \cite{rotskoff18}, \cite{chizat18}. 

The generalization capabilities and the definition of various quantities that measure these have been studied extensively. Previous work has shown that the norm \cite{bartlett17spec}, \cite{neyshabur15}, \cite{li18tight}, the width of a minimum in weight space \cite{hochreiter97}, \cite{sagun17}, the input sensitivity \cite{novak18} and a model's compressibility \cite{arora18} can be related (either theoretically or in practice) to the model's complexity and thus its ability to perform well on unseen data. 
Furthermore, it has been noted that the generalization capabilities can be influenced by the optimization algorithm used to train the model, e.g. it can be used to bias the model into configurations that are more robust to noise and have lower model complexity, see e.g. \cite{arora19}, \cite{srebro17}, \cite{neyshabur17}. 
Furthermore, it has been observed that certain parameters of stochastic gradient descent (SGD) can be used to control the generalization error and the data fit, see e.g. \cite{kenton17}, \cite{seong18}, \cite{borovykh19}, \cite{chaudhari18}, \cite{li19gen}, \cite{zhu19}, \cite{simsekli19}. Besides the optimization algorithm, also the particular model architecture can influence the generalization capabilities \cite{gunasekar18}, \cite{goldt19}.

Understanding the effects that the optimization scheme has on the output and weight evolution of deep neural networks is crucial for understanding what and how the model learns. The goal of this work is to present, \emph{using analytic formulas}, the effect of certain hyperparameters of the optimization algorithm, and in particular of the noise introduced in stochastic training, on the output and weight evolution of deep neural networks. 

In previous work \cite{jacot18} \cite{yang19} \cite{chizat18} it was shown that in the case of a neural network with infinitely wide layers and a particular scaling of the weights, the time-dependent kernel governing the weight evolution under gradient descent converges to a \emph{deterministic} kernel. The neural network output is in this case equivalent to that of a linear model (the first-order Taylor approximation of the network output around the weights) \cite{lee19}, and the output evolution can be solved for explicitly. The equivalence of a neural network trained with gradient descent to a linear model is also known as lazy training \cite{chizat18}, a regime under which the weights do not move far from their values at initialization. This regime can also be attained for finite width models when the scale at initialization goes to infinity \cite{chizat18}, \cite{woodworth19}. 

In this work we go beyond the linear model approximation and present a method that can be used to solve for the $n$-th order Taylor approximation of the deep neural network under a general form of stochastic training. In order to gain insight into the weight evolution, we can formulate a Cauchy problem which the density of the weights will satisfy under SGD training. This Cauchy problem has state-dependent coefficients, so an explicit solution is not trivial. In Section \ref{sec3}, using a Taylor expansion method \cite{lorig15a}, \cite{lorig17} we show how to obtain an analytic expression for an approximation of the expected value of the weights and network output under stochastic training. In the first phase of training, with a small noise variance and when the network weights are scaled in a particular way \cite{jacot18}, the network output can be approximated by a model which is linear in the weights. In Section \ref{sec32} we show that in this case the noise can result in a larger training error by keeping the network from fully converging on the train data, however the noise does not explicitly smooth or regularize the network output. 
In the case of a deep network with relatively narrow layers, one without the weight scaling or one which was trained for a large number of iterations, the network is no longer in the lazy training regime. In Section \ref{sec33} we study the influence of the hyperparameters of the optimization algorithm on the network output in this non-lazy setting. We show that a sufficient amount of noise can smooth the output function resulting in a smaller weight Hessian. We demonstrate the effects of the hyperparameters on weight and network evolution using numerical examples in Section \ref{sec4}. In the rest of this work, when we refer to `the amount of noise' we typically refer to the standard deviation $\sigma$ of noise present in the optimization algorithm. 

Our contributions can be summarized as follows: 1) we present a novel methodology for analyzing the output function under stochastic training when the first-order approximation is not sufficient -- which occurs for more training iterations or deep, relatively shallow, networks -- based on solving the Cauchy problem of the network output; 2) we show that under lazy training, i.e. when the network output is equivalent to the first-order Taylor approximation of the output around the weights, isotropic noise from noisy training does not smooth the solution, but keeps the loss function from fully converging; 3) we show that when the network is \emph{non-linear} in the weights (i.e. when the first-order approximation is not sufficient), noise has an explicit effect on the smoothness of the output, and sufficient and correctly scaled noise can decrease the complexity of the function. 

\section{A Taylor-expansion method for the network output}\label{sec2}
In this section we will introduce the neural network architecture and the optimization algorithms considered in this work. Furthermore, we will introduce a novel methodology for solving for the network output evolution under stochastic gradient descent through a Taylor expansion method. 
\subsection{Neural networks}\label{sec21}
Consider an input $x\in\mathbb{R}^{n_0}$ and a feedforward neural network consisting of $L$ layers with $n_l$ hidden nodes in each layer $l=1,...,{L-1}$ and a read-out layer with $n_{L}=1$ output node. 
Each layer $l$ in the network then computes for each $i=1,...,n_l$ 
\begin{align}\label{eq:nnstan}
a_i^l(x) =\sum_{j=1}^{n_{l-1}} w^l_{i,j}z^{l-1}_j + b^l_j,\;\;z_i^l(x)=h(a_i^l(x)),
\end{align}
where $h(\cdot)$ is the non-linear activation function, $w^l\in\mathbb{R}^{n_{l-1}\times n_{l}}$ and $b^l\in\mathbb{R}^{n_l}$. The weights and biases are initialized according to 
\begin{align}
w_{i,j}^l\sim\mathcal{N}\left(0,\frac{\sigma_w^2}{n_{l-1}}\right), \;\;\; b_{j}^l\sim\mathcal{N}\left(0,\sigma_b^2\right),
\end{align}
where $\sigma_w^2$, $\sigma_b^2$ are the weight and bias variances, respectively. 
This is the standard way of defining the neural network outputs. 
Alternatively, we can consider a different parametrization, which is referred to as \emph{NTK scaling}. It differs slightly from the standard one in \eqref{eq:nnstan} due to the scaling of the weights in both the forward and backward pass (see also Appendix E in \cite{lee19}),
\begin{align}\label{eq:nnntk}
a_i^l(x) = \frac{\sigma_w}{\sqrt{n_{l-1}}}\sum_{j=1}^{n_{l-1}} w^l_{i,j}z^{l-1}_j + \sigma_b b^l_j,\;\;z_i^l(x)=h(a_i^l(x)).
\end{align}
Note that in both cases in the first layer we compute the linear combination using the input, i.e. $z^0 = x$ with $n_0$ as the input layer size. The output is given by $\hat y(x) = a^{L+1}(x)$. Let $X\in\mathbb{R}^{N\times n_0}$ with $X=(x^1,...,x^N)$, and $Y\in\mathbb{R}^{N\times 1}$ with $Y=(y^1,...,y^N)$ be the training set.

Let the loss function of the neural network over the data $(X,Y)$ be given by $\mathcal{L}(X,\hat y_t,Y)$, where we will drop the dependence on $(X,Y)$ if this is clear from the context. In regression problems the loss is given by the mean-squared error:
\begin{align}
\mathcal{L}(X,\hat y_t,Y) = \frac{1}{2N}||\hat y_t(X) - Y||_2^2 = \frac{1}{2N}\sum_{i=1}^N \left(\hat y_t(x^i)-y^i\right)^2.
\end{align}
One way of minimizing the neural network loss function is through gradient descent. Set $\theta^l = [w_{i,j}^l,b^l_j]_{i,j}$ to be the collection of parameters mapping to the $l$-th layer such that $\theta^l\in\mathbb{R}^{(n_{l-1}+1)\times n_l}$ and let $\theta\in\mathbb{R}^d$ where $d=(n_0+1)n_1+...(n_{L-1}+1)n_L$ be the vectorized form of all parameters. The gradient descent updates are then given by,
\begin{align}
\theta_{t+1}=\theta_t-\eta\nabla_\theta\mathcal{L}(\hat y_t), \;\;\;\textnormal{where }\nabla_\theta\mathcal{L}(\hat y_t)= (\nabla_\theta\hat y_t)^T\nabla_{\hat y}\mathcal{L}(\hat y_t).
\end{align}
By a continuous approximation of the discrete gradient descent updates (see e.g. \cite{li17} for when this approximation holds) the parameters evolve as, 
\begin{align}\label{eq:nngd}
\partial_t \theta_t &= -\eta \nabla_\theta \mathcal{L}(\hat y_t),
\end{align}
where $\eta$ is the learning rate and  $\nabla_\theta\hat y_t\in\mathbb{R}^{N\times d}$. 
In gradient descent the gradient is computed over the full training data batch. Alternatively, one can use \emph{stochastic} gradient descent, where the gradient is computed over a subsample of the data. Let the gradient in a mini-batch $\mathcal{S}$ be $g_\mathcal{S}\in\mathbb{R}^d$ and the full gradient be $g\in\mathbb{R}^d$, where $d$ is the weight space dimension, defined respectively as,
\begin{align}
g_{\mathcal{S}} := \frac{1}{M}\sum_{i\in\mathcal{S}}\nabla_{\theta}\mathcal{L}(x^i,\hat y_t,y^i)=:\frac{1}{M}\sum_{i\in\mathcal{S}}g_i.
\end{align}
The weight update rule is given by $\theta_{t+1} = \theta_t - \eta g_{\mathcal{S}}$, where $\eta$ is the learning rate. By the central limit theorem, if the train data $(x_i,y_i)\sim\mathcal{D}$ i.i.d.
the weight update rule can be rewritten as,
\begin{align}\label{eq:wupnoise}
\theta_{t+1}=\theta_t - \eta \nabla_\theta\mathcal{L}(\hat y_t)+ \frac{\eta}{\sqrt{M}} \epsilon,
\end{align}
where $\epsilon\sim\mathcal{N}(0,D(\theta))$ with $\epsilon\in\mathbb{R}^d$ and $D(\theta)=\mathbb{E}\left[( g_{\mathcal{S}}-g)^T(g_{\mathcal{S}}-g) \right]$.

This expression can then be considered to be a finite-difference equation that approximates a continuous-time SDE. If we assume that the diffusion matrix $D(\theta)$ is isotropic, i.e. it is proportional to the identity matrix and let $D(\theta) = \sigma^2 I_{d}$, the continuous-time evolution is given by
\begin{align}\label{eq:sgdsde}
d\theta_t &= -\eta \nabla_\theta\mathcal{L}(\theta_t)dt+\sigma \frac{\eta}{\sqrt{M}}dW_t=-\eta\nabla_\theta\hat y_t^T \nabla_{\hat y}\mathcal{L}(\hat y_t)dt + \sigma\frac{\eta}{\sqrt{M}}dW_t.
\end{align}
In this work we will focus on understanding the effects of the noise, given by a scaled Brownian motion, on the evolution of the network weights and output. 
The convergence to the continuous-time SDE is studied in e.g. \cite{bartlett19}. The continuous limit of SGD has been analysed in prior work \cite{mandt17}, \cite{li17}, \cite{simsekli19}. 

Studying the behavior of the above SDE is of interest in order to gain insight into the network output evolution. Under mild regularity assumptions on $\mathcal{L}$, the Markov process $(\theta_t)_{t\geq 0}$ is ergodic with a unique invariant measure whose density is proportional to $\displaystyle e^{-\mathcal{L}(x)M/\eta^2}$ \cite{roberts02}. Analyzing the network output based on the \emph{invariant} measure implicitly assumes that sufficient iterations have been taken to converge to that measure, which for a finite number of iterations is not always the case. The invariant measure thus cannot be used to explain the behavior of SGD and its generalization ability for \emph{finite} training times. In the next sections we analyze the behavior of the network weights and output through analytic formulas as a function of time $t$, so that also the finite-time behavior can be studied.  

\subsection{Sufficiency of the linear model}
Consider a FNN whose parameters $\theta$ are trained with gradient descent as in \eqref{eq:nngd}. The network output evolves as 
\begin{align}
\label{eq:evolution2}
\partial_t \hat y_t &= \nabla_\theta \hat y_t\partial_t \theta_t=-\eta \nabla_\theta\hat y_t(\nabla_\theta \hat y_t)^T\nabla_{\hat y}\mathcal{L}(\hat y)=:-\eta \Theta^L_t \nabla_{\hat y}\mathcal{L}(\hat y),
\end{align}
where $\Theta_t^L\in\mathbb{R}^{N\times N}$ is the defined to be the neural tangent kernel (NTK) for the output layer $L$.  For simplicity we denote $\Theta_t^L:=\Theta_t$. In this section we will repeat the conclusions of previous work \cite{jacot18}, \cite{lee19} which shows that under the NTK scaling in \eqref{eq:nnntk} the network output is equivalent to a model which is linear in the weights and one can solve explicitly for the network output. Observe that in general the NTK from \eqref{eq:evolution2} is random at initialization and varies during training; however as shown by the authors of \cite{jacot18} as the number of neurons in an FNN sequentially goes to infinity $n_1,...,n_{L-1}\rightarrow\infty$ (note that \cite{yang19} extends this to a simultaneous limit) the NTK converges to a deterministic limit. In particular we have,
\begin{lemma}[NTK convergence \cite{jacot18}]\label{lem0}
If $n_1,...,n_{L-1}\rightarrow\infty$ sequentially and
\begin{align}\label{eq:condsb}
\int_0^T||\nabla_{\hat y}\mathcal{L}(\hat y_t(X))||_2dt<\infty,
\end{align} 
 the NTK converges to a deterministic limit $\Theta_t\rightarrow\Theta_0$, which can be computed recursively using,
\begin{align}
&\Theta_0^0(x,x') = k^0(x,x') \otimes I_{n_1},\\
&\Theta_0^{l}(x,x') = (\sigma_w^2\Theta_0^{l-1}(x,x')\dot k^{l}(x,x')+k^l(x,x'))\otimes I_{n_l},
\end{align}
where $k^l(x,x') = \mathbb{E}\left[h(a^{l-1}(x))h(a^{l-1}(x'))\right]$ and $\dot k^l(x,x') = \mathbb{E}\left[h'(a^{l-1}(x))h'(a^{l-1}(x'))\right]$. 
\end{lemma}
The evolution of the tangent kernel is independent of the values of the parameters at time $t$, i.e. time-independent. This makes the evolution of the network equivalent to a linear network as proven in \cite{lee19}. 
\begin{lemma}[Linear network similarity \cite{lee19}]\label{lem01}
Define
\begin{align}\label{eq:linappgd}
\hat y^{lin}_t := \hat y_0 + \nabla_\theta \hat y_0 (\theta_t-\theta_0).
\end{align}
Under the assumptions from Lemma \ref{lem0} it holds that $\sup_{t\in [0,T]}||\hat y_t-\hat y^{lin}_t||_2\rightarrow 0$ as $n_1,...,n_{L-1}\rightarrow\infty$ sequentially.
\end{lemma}
\begin{proof}
In general to show that a network output can be approximated by some function one has to show that the network output at initialization is equal to the value of the approximating function and that the output during training does not deviate from the evolution of the approximation. Informally, the above proof on the sufficiency of the linear approximation follows from the fact that the training dynamics of the original network do not deviate far from the training dynamics of the linear network. This in turn holds under the NTK convergence, i.e. if $\sup_{t\in[0,T]}||\Theta_t-\Theta_0||_{op}\rightarrow 0$. For the full proof we refer to Appendix E in \cite{lee19}.
\end{proof}
From now on we use the notation for $\hat y_t$ and $\hat y^{lin}_t$ interchangeably if these are equivalent. 
Using the fact that $\hat y_t$ can be approximated by a linear model, we obtain as an approximation to \eqref{eq:nngd} and \eqref{eq:evolution2},
\begin{align}\label{eq:linupd}
\partial_t\theta_t = -\eta (\nabla_\theta\hat y_0)^T\nabla_{\hat y}\mathcal{L}(\hat y_t),\;\;\;\partial_t\hat y_t(X) = -\eta \Theta_0^L \nabla_{\hat y}\mathcal{L}(\hat y_t),
\end{align}
which can be solved to give
\begin{align}\label{eq:sols1}
&\theta_t = \theta_0 - \nabla_\theta\hat y_0(X)^T(\nabla_\theta\hat y_0(X)\nabla_\theta\hat y_0(X)^T)^{-1}(I-e^{-\frac{\eta}{N} \nabla_\theta\hat y_0(X)\nabla_\theta\hat y_0(X)^T t})(\hat  y_0(X)-Y),\\\label{eq:sols2}
&\hat y_t(X) = \left(I-e^{-\frac{\eta}{N} \nabla_\theta\hat y_0(X)\nabla_\theta\hat y_0(X)^Tt}\right)Y+e^{-\frac{\eta}{N} \nabla_\theta\hat y_0(X)\nabla_\theta\hat y_0(X)^Tt}\hat y_0,
\end{align}
where as $n_1,...,n_{L-1}\rightarrow\infty$, sequentially, we have $\nabla_\theta\hat y_0(X)(\nabla_\theta\hat y_0(X))^T\rightarrow\Theta_0$. Observe that, due to the weights being normally distributed at initialization, $\hat y_t(X)$ is thus given by a Gaussian process. We can use the obtained form of the output function in order to gain insight into certain properties of the trained network, such as its generalization capabilities as we will do in Section \ref{sec3}. 

\subsection{The Taylor expansion method}\label{sec23}
In the previous section, under the first-order approximation of the neural network the output under gradient descent (and under stochastic gradient descent, as we will show later on) can be obtained explicitly. In general however this first-order approximation is not sufficient. Under gradient descent, if the number of iterations increases and the network is relatively deep and shallow (even under NTK scaling) higher order approximations are needed. Under stochastic gradient descent, if the noise is sufficiently large, the linear approximation is also no longer sufficient. 

In this section we briefly present the Taylor expansion method \cite{lorig15a}, \cite{lorig17} that will be used to obtain an analytic approximation for the expected value of the weights and network output. The results are presented for a general SDE with state-dependent coefficients, i.e. with the drift and volatility coefficients depending on the SDE itself. 
Consider, 
\begin{align}
dX_t = \mu(t,X_t)dt + \sigma(t,X_t)dW_t.
\end{align}
Here $X$ is a diffusion process with local drift function $\mu(t,x)$ and local volatility function $\sigma(t,x)\geq 0$. 
Assume that we are interested in computing the following function,
\begin{align}\label{eq:fny}
u(t,x) := \mathbb{E}[h(X_T)|X_t=x],
\end{align}
and $h(\cdot)$ is some known function. The expectation is thus taken over the evolution of $X_t$ given the initial value $x$ at time $t$. Note that we obtain the density by setting $h(X_t) = \delta_y(X_t)$, i.e.
\begin{align}
\Gamma(t,x;T,y) = \mathbb{E}[\delta_y(X_T)|X_t=x],
\end{align} 
where $\Gamma(t,x;T,y)dy = \mathbb{P}[X_T\in dy|X_t=x]$. By a direct application of the Feynman-Kac representation theorem, the classical solution of the following Cauchy problem,
\begin{align}\label{eq:cauchy1}
(\partial_t+\mathcal{A}) u(t,x) =0, \;\;\; u(T,x) = h(x),
\end{align}
when it exists, is equal to the function $u(t,x)$ in \eqref{eq:fny}, where $\mathcal{A}$ is the generator related to the SDE and is defined as,
\begin{align}
\mathcal{A} u(t,x) = \mu(t,x)\partial_xu(t,x) + \frac{1}{2}\sigma^2(t,x)\partial_x^2u(t,x).
\end{align}
Due to the state-dependency, the above Cauchy-problem in \eqref{eq:cauchy1} is not explicitly solvable. We follow the work of \cite{lorig15a} and consider solving it through a Taylor-expansion of the operator $\mathcal{A}$. The idea is to choose a (polynomial) expansion $(\mathcal{A}_n(t))_{n\in\mathbb{N}}$ that closely approximates $\mathcal{A}(t)$, i.e. 
\begin{align}\label{eq:exp1}
\mathcal{A}(t,x) = \sum_{n=0}^\infty \mathcal{A}_n(t,x).
\end{align}
The precise sense of this approximation, e.g. pointwise local or global, will depend on the application. The simplest approximation is given by a Taylor expansion of the coefficients of the operator, so that we have
\begin{align}
\mathcal{A}_n(t,x) = \mu_n(t,x)\partial_x + \frac{1}{2}\sigma^2_n(t,x)\partial_x^2,
\end{align}
and 
\begin{align}
&\sigma^2_n(t,x)=\frac{\partial_x^n\sigma^2(t,\bar x)}{n!}(x-\bar x)^n,\;\;\;\mu_n(t,x)=\frac{\partial_x^n\mu(t,\bar x)}{n!}(x-\bar x)^n.
\end{align}
This Taylor expansion can be used under the assumption of smooth, i.e. differentiable, coefficients. 
Following the classical perturbation approach, the solution $u$ can be expanded as an infinite sum,
\begin{align}\label{eq:exp2}
u=\sum_{n=0}^\infty u^n.
\end{align}
The $N$-th order approximation of $u$ is then given by,
\begin{align}
u^{(N)}(t,x) = \sum_{n=0}^N u^n(t,x).
\end{align}
Inserting \eqref{eq:exp1}-\eqref{eq:exp2} into \eqref{eq:cauchy1}, we obtain the following sequence of nested Cauchy problems, for $x\in\mathbb{R}$,
\begin{align}\label{eq:cauchy2}
&(\partial_t+\mathcal{A}_0)u^0(t,x) = 0, \;\;\; u^0(T,x) = h(x),\\
&(\partial_t+\mathcal{A}_0)u^n(t,x) = -\sum_{k=1}^n\mathcal{A}_k(t,x)u^{n-k}(t,x), \;\;\; u^n(T,x) = h(x), \; n>0.
\end{align}
By the results in \cite{lorig15a} we then have the following result for the solution of the Cauchy problem,
\begin{theorem}[Solution for a one-dimensional SDE \cite{lorig15a}]\label{thm1}
Given the Cauchy problem in \eqref{eq:cauchy1}, assuming that $\mu,\sigma\in C^{n}(\mathbb{R})$, 
we obtain for $n=0$,
\begin{align}\label{eq:zero1d}
&u^0(t,x) = \int_\mathbb{R} G_0(t,x;T,y)h(y)dy,\\
&G_0(t,x;T,y)=\frac{1}{ \sqrt{2\pi \int_t^T a_0(s)ds}}\exp\left(-\frac{(y-x-\int_t^T\mu_0(s))^2}{2\int_t^Ta_0(s)ds}\right),
\end{align} 
and the higher order terms for $n>0$ satisfy,
\begin{align}\label{eq:napprox}
&u^n(t,x) = \mathcal{L}_n(t,T) u^0(t,x),\\
&\mathcal{L}_n(t,T) = \sum_{k=1}^n\int_t^Tdt1\int_{t_1}^Tdt_2\dots\int_{t_{k-1}}^Tdt_k\sum_{i\in I_{n,k}}\mathcal{G}_{i_1}(t,t_1)\mathcal{G}_{i_2}(t,t_2)\dots \mathcal{G}_{i_k}(t,t_k),
\end{align}
where 
\begin{align}
I_{n,k} = \{i=(i_1,...,i_k)\in\mathbb{N}^k|i_1+\dots +i_k=n\},
\end{align}
and $\mathcal{G}_n(t,s)$ is an operator
\begin{align}
\mathcal{G}_n(t,s)=\mathcal{A}_n\left(t,x+\int_t^s\mu_0(u)du+\int_t^s a_0(u)du\partial_x\right).
\end{align}
\end{theorem}
We can observe that the zeroth-order approximation of the density is given by a Gaussian, and the subsequent orders are given by differential operators applied to the Gaussian density. 
For example, the first-order approximation is given by $\Gamma_1(t,x;T,y) = \int_t^Tds \mathcal{G}_1(t,s)\Gamma_0(t,x;T,y)$. Note that, $\Gamma_0(t,x;T,y)$ is the unique non-rapidly increasing solution of the backward Cauchy problem $(\partial_t+\mathcal{A}_0)\Gamma_0(t,x;T,y)=0$ with $\Gamma_0(T,x;T,y) = \delta_y(x)$. This in turn corresponds to the density of $dX^0_t = \mu_0 dt + \sigma_0 dW_t$. 


\begin{remark}[Convergence of the approximation]\label{remarkconv}
As obtained in Theorem 4.4 in \cite{lorig15a}, under certain regularity and smoothness assumptions on the coefficients, for small times $T$ convergence as $N\rightarrow\infty$ can be obtained. 
This result does not imply convergence for larger times $T$ due to the possibility of the bound exploding in the limit as $N\rightarrow\infty$. However, in e.g. \cite{lorig15a}, \cite{lorig17} the authors use the approximation for large times $T$ and obtain accurate results. Nevertheless, the approximation works best for times close to $t_0$, i.e. close to initialization of the network. The further away from initialization the more approximation terms are needed to obtain satisfactory results. 
\end{remark}

\subsection{An approximation of the network output}\label{sec24}
In this section we focus on deriving an analytic expression for the network weights and output evolution using the Taylor expansion method presented in Section \ref{sec23}. We focus here on deriving expressions for the scalar case of the weights, i.e. for $d=1$ so that $\theta_t\in\mathbb{R}$. In this case, the $N$-th order Taylor expansion of the network output around $\bar\theta$ can be written as, 
\begin{align}\label{eq:taylor1d}
\hat y_t^{(N)}(X) := \sum_{n=0}^N \frac{\partial_\theta^n \hat y_t(X)}{n!}\bigg |_{\bar\theta}(\theta_t-\bar\theta)^n.
\end{align}
The aim is to obtain an analytic expression for the expected value of this $N$-th order approximation. This requires solving for $\mathbb{E}[(\theta_t-\bar\theta)^m]$, $m=0,...,N$. 
We assume that the evolution of $\theta_t$ is governed by the following continuous approximation of the stochastic optimization dynamics,
\begin{align}\label{eq:sdeevol1}
d\theta_t = -\eta (\partial_\theta\hat y_t)^T \nabla_{\hat y}\mathcal{L}(\hat y_t)dt + \sigma dW_t.
\end{align}
Note that this is similar to \eqref{eq:sgdsde} but with a general noise component $\sigma$. In general, the above SDE, due to the intricate dependencies in the drift term, is not directly solvable. By making use of the approximation method as presented in Section \ref{sec23} we can obtain expressions for the expected network output.
Define, 
\begin{align}\label{eq:mudef}
\mu(\theta_t) := -\eta(\partial_\theta\hat y_t(X))^T\nabla_{\hat y}\mathcal{L}(\hat y_t)=-\frac{\eta}{N}(\partial_\theta\hat y_t(X))^T(\hat y_t(X)-Y).
\end{align}
Define,
\begin{align}\label{eq:weightapprox}
u_m(t_0,\theta):=\mathbb{E}\left[(\theta_t-\bar\theta)^m\right|\theta_{t_0}=\theta],
\end{align}
Note again that the expected value is taken over the evolution of $\theta_t$ given a particular initialization $\theta_{t_0}$. We have, similar to \eqref{eq:cauchy1},
\begin{align}
(\partial_t+\mathcal{A}) u_m(t,\theta) =0, \;\;\; u_m(T,\theta) = (\theta-\bar\theta)^m,
\end{align}
where the generator for the SDE in \eqref{eq:sdeevol1} is given by,
\begin{align}
\mathcal{A}(t,\theta) = \mu(\theta)\partial_\theta + \frac{1}{2}\sigma^2\partial_\theta^2,
\end{align}
with $\mu(\cdot)$ as in \eqref{eq:mudef}. Similar to the previous section we want to find an approximation of the generator of the form,
\begin{align}\label{eq:genapp}
\mathcal{A}(t,\theta) \approx \sum_{n\geq 0}\mathcal{A}_n:=\sum_{n\geq 0} \mu_n(\theta)\partial_\theta + \caratt_{n=0}\frac{1}{2}\sigma\partial^2_\theta,
\end{align}
with $\mu_n(\theta):=\bar\mu_n (\theta-\bar\theta)^n$. 
In this case, instead of applying a direct Taylor approximation to the coefficient $\mu(\theta)$, we use the Taylor expansion of the network output in \eqref{eq:taylor1d} in order to obtain a polynomial expansion of the above form for the generator. Inserting \eqref{eq:taylor1d} in \eqref{eq:mudef}, we obtain 
\begin{align}\label{eq:muexp2}
\mu_n(\theta) := 
&-\frac{\eta}{N}\sum_{\substack{k\leq N-1,j\leq N:\\
                  k+j=n}} \frac{1}{k!}(\partial_\theta^{k+1}\hat y_t(X)|_{\bar\theta} )^T\left(\frac{1}{j!}\partial_\theta^{j}\hat y_t(X)|_{\bar\theta}-\caratt_{k=n}Y\right)(\theta-\bar\theta)^n, \;\; n=0,...,2N-1.
\end{align}
Note that in $\hat y_t$ the time-dependency arises from $\theta_t$, so that $\hat y_t|_{\bar\theta}$ defines the time-fixed output with weights fixed at $\bar\theta$. Define the approximation of \eqref{eq:weightapprox} corresponding to the $N$-th order expansion in $\hat y$ as 
\begin{align}
u_m^{(N)}=\sum_{n=0}^{2N-1} u_m^n,
\end{align} 
 Then, we have the following set of Cauchy problems for $u_m^n$, $n=0,...,2N-1$,
\begin{align}\label{eq:cauchynn}
&(\partial_t+\mathcal{A}_0)u_m^0 = 0, \;\;\; u_m^0(T,\theta)=(\theta-\bar\theta)^k,\\
&(\partial_t+\mathcal{A}_0)u_m^n = -\sum_{k=1}^n\mathcal{A}_k u_m^{n-k}, \;\;\; u_m^n(T,\theta)=0, \;\; n=1,...,2N-1.
\end{align}
To summarize, suppose that one wants to obtain an expression for the $N$-th order approximation of the expected value of $\hat y_t$. This requires the approximation of $u_0,...,u_N$. For each $u_k$ one then has to solve $2N-1$ Cauchy problems, as given by \eqref{eq:cauchynn}. The difference with Section \ref{sec23} is thus that we do not directly expand the generator, but expand the generator through an expansion of $\hat y$. This means that for different $N$, the $\mu_0,...,\mu_{2N-1}$ will be different; having fixed $N$, the obtained solutions $u_m^0,...,u_m^{2N-1}$ cannot be reused for a different $N$. 
Since the methodology presented in Section \ref{sec24} holds for a general expansion basis, we can apply the results of Theorem \ref{thm1} to obtain the following result. 

\begin{theorem}[Analytic expressions for the weights and output]\label{cor1}
Fix $N\in\mathbb{N}$ and consider the $N$-th order approximation of $\hat y_t$. The expected value of $\hat y^{(N)}_t$ is then given by, 
\begin{align}
\mathbb{E}\left[\hat y_t^{(N)}(X)|\hat y_{0}\right] = \sum_{n=0}^N\frac{\partial_\theta^n\hat y_t(X)}{n!}\bigg |_{\bar\theta}u_m^{(N)}(t_0,\theta),
\end{align} 
where the corresponding $2N-1$-th order approximation of $u_m$ is given by,
\begin{align}
u_m^{(N)}=\sum_{n=0}^{2N-1} u_m^n,
\end{align}
with
\begin{align}
&u_m^0(t_0,\theta) = \partial_s^k\exp\left(\left(-\frac{\eta}{N}\mu_0 (t-t_0) +\theta-\bar\theta\right)s+\frac{1}{2}\sigma^2(t-t_0)s^2\right)\big|_{s=0},\\
&u^{n}_m(t_0,\theta)=\mathcal{L}_n(0,t) u_m^0(t_0,\theta).
\end{align}
\end{theorem}
\begin{proof}
Observe that we have by the result in Theorem \ref{thm1} it follows that, 
\begin{align}
u^0_m(t_0,\theta) &= \mathbb{E}\left[(\theta_t^0-\bar\theta)^k|\theta_{t_0}\right],
\end{align}
where $\theta_t^0$ follows a Gaussian distribution with mean $\theta-\frac{\eta}{N}\mu_0 (t-t_0)$ and variance $\sigma^2(t-t_0)$. Then, $\theta_t^0-\bar\theta$ follows a Gaussian distribution with mean $\theta-\bar\theta-\frac{\eta}{N}\mu_0 (t-t_0)$ and variance $\sigma^2(t-t_0)$. 
Using the fact that the Gaussian moments are known through the characteristic exponential, we obtain
\begin{align}
u^0_m(t_0,\theta) = \partial_s^me^{\left(-\frac{\eta}{N}\mu_0 (t-t_0) +\theta-\bar\theta\right)s+\frac{1}{2}\sigma^2(t-t_0)s^2}\big|_{s=0}.
\end{align}
The expression for the higher order terms then follows from \eqref{eq:napprox}.
\end{proof}
\begin{example}[The second-order approximation of $\hat y_t$]\label{ex1}
In this example we present the analytic expression for the second-order approximation of $\hat y_t$:
\begin{align}
\hat y_t^{(2)} = \hat y_t(X)|_{\bar\theta} + \partial_\theta\hat y_t(X)|_{\bar\theta}(\theta_t-\bar \theta)+\frac{1}{2}\partial^2_\theta\hat y_t(X)|_{\bar\theta}(\theta_t-\bar \theta)^2.
\end{align}
The corresponding approximation of the generator is given by \eqref{eq:genapp} with,
\begin{align}
&\mu_0= -\frac{\eta}{N}\partial_\theta \hat y^T(\hat y-Y),\;\;
\mu_1(\theta) =\bar\mu_1(\theta-\bar\theta):=- \frac{\eta}{N}\left(\partial^2_\theta \hat y^T(\hat y-Y)+\partial_\theta\hat y ^T\partial_\theta\hat y\right)(\theta-\bar\theta),\\
&\mu_2(\theta) =\bar\mu_2(\theta-\bar\theta)^2:= -\frac{\eta}{N}\left(\frac{1}{2}\partial_\theta\hat y^T\partial_\theta^2\hat y+\partial_\theta^2\hat y^T\partial_\theta\hat y\right)(\theta-\bar\theta)^2,\;\;\mu_3(\theta) =\bar\mu_3(\theta-\bar\theta)^3:= -\frac{\eta}{N}\frac{1}{2}\partial_\theta^2\hat y^T\partial_\theta^2\hat y(\theta-\bar\theta)^3,
\end{align}
where we have used the shorthand notation $\hat y$ to represent $\hat y_t(X)|_{\bar\theta}$. 
For the second-order approximations of the first two moments, we can obtain the following results:
\begin{align}
u^0_1(t_0,\theta) =& \theta-\bar\theta-\bar\theta_0(t-t_0),\\
u^1_1(t_0,\theta) =& \bar\mu_1(\theta-\bar\theta)(t-t_0)+\frac{1}{2}\bar\mu_1\mu_0(t-t_0)^2,\\
u^2_1(t_0,\theta)=&\bar\mu_2(\theta-\bar\theta)^2(t-t_0)+\bar\mu_2\mu_0(\theta-\bar\theta)(t-t_0)^2+\frac{1}{3}\mu_0^2(t-t_0),\\
u^0_2(t_0,\theta) =& (\theta-\bar\theta+\bar\mu_0(t-t_0))^2+\sigma^2(t-t_0),\\
u^1_2(t_0,\theta) =& 2\bar\mu_1(\theta-\bar\theta)^2(t-t_0)-2\bar\mu_1(\theta-\bar\theta)\mu_0(t-t_0)^2+\frac{1}{2}\bar\mu_1\mu_0(t-t_0)^2(\theta-\bar\theta)\\
&-\frac{1}{2}\bar\mu_1\mu_0^2(t-t_0)^3+\frac{1}{2}\bar\mu_1\sigma^2(t-t_0)\\
u^2_2(t_0,\theta)=&2\bar\mu_2(\theta-\bar\theta)^3(t-t_0)+2\bar\mu_2\mu_0(\theta-\bar\theta)^2(t-t_0)^2+2\bar\mu_2\mu_0(\theta-\bar\theta)^2(t-t_0)^2+2\bar\mu_2\mu_0^2(\theta-\bar\theta)(t-t_0)^3\\
&+\frac{2}{3}\mu_0^2(t-t_0)(\theta-\bar\theta) +\frac{2}{3}\mu_0^3(t-t_0)^2 + \bar\mu_2(\theta-\bar\theta)\sigma^2(t-t_0)^2+\frac{2}{3}\mu_0\sigma^2(T-t)^3+\bar\mu_1^2(\theta-\bar\theta)^2(t-t_0)\\
&+\bar\mu_1^2\mu_0(\theta-\bar\theta)(t-t_0)^3+\frac{1}{3}\bar\mu_1(\theta-\bar\theta)\sigma^2(t-t_0)^3+\frac{1}{4}\mu_1^2\mu_0^2(t-t_0)^4+\frac{1}{8}\bar\mu_1^2\mu_0\sigma^2(t-t_0)^4.
\end{align}
The third-order formulas are too long to include but can easily be computed in a similar way.
\end{example}

\section{The effects of the optimization algorithm}\label{sec3}
In this section, using the analytical expressions for the network output evolution obtained in the previous section we study the effects of the optimization algorithm, in particular that of the number of training iterations, the learning rate and the noise variance from the stochastic optimization, on the network output and the networks' generalization capabilities. 
\subsection{Generalization}\label{sec31}
Generalization is the relationship between a trained networks' performance on train data versus its performance on test data. Having good generalization is a highly desirable property for neural networks, where ideally the performance on the train data should be similar to the performance on similar but unseen test data. In general, the generalization error of a neural network model $\hat y(x,w)$ can be defined as the failure of the hypothesis $\hat y(x,w)$ to explain the dataset sample. It is measured by the discrepancy between the true error $\mathcal{L}$ and the error on the sample dataset $\mathcal{\hat L}$, 
\begin{align}\label{eq:genbound}
\mathcal{L}(\hat y(X)) - \mathcal{\hat L}(\hat y(X)),
\end{align}
which for good generalization performance should be small. Typically, a trained network is able to generalize well when it is not overfitting on noise in the train dataset. Various metrics for studying the networks ability to generalize have been defined. 
 
One way of analyzing the generalization capabilities of the neural network is to understand the evolution in the weight space. In particular, if we assume that flat minima generalize better \cite{hochreiter97}, an algorithm that converges to these flat minima will have better out-of-sample performance than one that converges to a sharp one. In the work of \cite{zhu19} it is shown that isotropic noise is not sufficient for escaping sharp minima; a noise covariance matrix proportional to the loss curvature is shown to be better at escaping sharp minima. The authors of \cite{simsekli19} use a metastability analysis to give insight into the structure of the minima where the optimization algorithm terminates. An alternative way of analyzing the SGD trajectories is to use the invariant distribution of the SGD SDE, see e.g. \cite{mandt17}. The downside of this approach is that it assumes that the algorithm has converged to this invariant distribution, which might not be the case.

In this work, we use the idea that good generalization is related to the network having a sufficient robustness to noise. In this setting, the derivatives of the network output or the loss function have been proposed as empirical metrics that measure the smoothness and noise resistance of the function, see e.g. \cite{borovykh19}, \cite{novak18}, \cite{sokolic17}. 
One such metric is the Hessian $H^\theta(\mathcal{L})\in\mathbb{R}^{d\times d}$ of the loss function with respect to the weights, with elements $h_{ij}^\theta=\partial_{w_i}\partial_{w_j}\mathcal{L}(\hat y(X))$. This measures the flatness of the minimum in weight space and can be related to theoretical quantities that measure generalization capabilities such as low information \cite{hochreiter97} or the Bayesian evidence \cite{smith18}. Alternatively, the input Jacobian $J^x(\hat y)$ with elements $j_{i}^x= \partial_{x_i} \hat y(x)$, or input Hessian $H^x(\mathcal{L})\in\mathbb{R}^{n_0\times n_0}$ with $h_{ij}^x = \partial_{x_i}\partial_{x_j} \mathcal{L}(\hat y(X))$ can be used to gain insight into input noise robustness of the network \cite{novak18}. These metrics are related to the smoothness of the output function. Obtaining good generalization is then related to a correct trade-off between the smoothness of the neural network output function and its ability to fit the training data (see e.g. \cite{tishby15}).

The generalization capabilities of a network are influenced by the training data, the network architecture and the optimization algorithm. In particular related to the latter, previous work has observed that the noise in noisy training can contribute to better generalization capabilities \cite{kenton17}, \cite{seong18}, \cite{borovykh19}, \cite{chaudhari18}, \cite{li19gen}, \cite{zhu19}, \cite{simsekli19}. In the remainder of this work we will focus on studying the effect noise has on the train error and function smoothness, as measured by the derivatives with respect to input or weights. 

\subsection{The first-order approximation}\label{sec32}
In this section we present the results of the optimization algorithm on the network output during full-batch gradient descent and stochastic training in case the linear model is a sufficient approximation. The linear model output during training allows to be solved for explicitly without the need for the Taylor expansion method. 
\subsubsection{Convergence speed and generalization capabilities}
Consider the convergence speed with which $\hat y_t(X)$ converges to the true labels $Y$. Let $\Theta_0=\sum_{i=1}^N \lambda_i v_iv_i^T$, and let $v_1,...,v_N\in\mathbb{R}^N$ be the orthonormal eigenvectors of $\Theta_0$ and $\lambda_1,...,\lambda_N$ the corresponding eigenvalues. Similar to the results obtained in \cite{arora19} for a two-layer network, one can obtain, through the eigen-decomposition of $\Theta_0$, the following result:
\begin{lemma}[Convergence speed]\label{lem1} 
Consider an $n^*$ such that $||\hat y^{lin}-\hat y||<\epsilon$ for some small enough $\epsilon$ if $n_1,...,n_{L-1}>n^*$. Assume $\lambda_0:=\lambda_{\min}(\Theta_0^L)>0$. 
\begin{align}\label{eq:convterm}
||\hat y_t(X)-Y||_2^2 = \sum_{i=1}^N e^{-2\frac{\eta}{N}\lambda_i t} \left(v_i^T \left(\hat y_0 - Y\right)\right)^2.
\end{align}
\end{lemma}
\begin{proof}
Note that from \eqref{eq:sols2} it holds,
\begin{align}
\hat y_t(X) - Y = e^{-\frac{\eta}{N} \Theta_0 t}\left(\hat y_0 - Y\right).
\end{align}
Observe that the exponential has the same eigenvectors as $\Theta_0$ but with eigenvalues given by $e^{-\frac{\eta}{N}\lambda_i t}$ instead of $\lambda_i$. Then we find, using the decomposition of $(\hat y_0-Y)$,
\begin{align}
\hat y_t(X) - Y = \sum_{i=1}^N e^{-\frac{\eta}{N}\lambda_i t} \left(v_i^T \left(\hat y_0 - Y\right)\right)v_i.
\end{align}
The statement then follows by taking the squared norm. 
\end{proof}
The MSE is thus a function of the eigenvalues $\lambda_i$, the iteration number $t$ and the projection of the eigenvectors on the labels $v_i^T Y$. The convergence speed is governed by the right-hand side: the faster the right-hand term converges to zero as $t$ increases, the faster the convergence of $\hat y(X)$ to $Y$. Convergence is faster along the directions corresponding to the larger eigenvalues.

In previous work, e.g. \cite{srebro17}, \cite{srebro17opt}, it was shown that gradient descent performs some form of implicit regularization. Due to this, the solution obtained by gradient descent generalizes well, since it can be shown to be the lowest-complexity solution in some notion of complexity. More specifically, in \cite{srebro17}, the authors show that optimizing a matrix factorization problem with gradient descent with small enough step sizes converges to the minimum nuclear norm solution.
Similarly, in \cite{zhang16} it is shown that (S)GD on a \emph{linear} model converges to a solution which has the minimum $L_2$-norm. Gradient descent on the matrix factorization and the linear network thus contains an implicit regularization, resulting in a solution that is minimal in some norm. Here we show that the solution to which gradient descent converges on a deep and infinitely wide neural network is also minimal in terms of the $L_2$ norm.
\begin{lemma}[Minimum norm solution]\label{lem02}
Consider an $n^*$ such that $||\hat y^{lin}-\hat y||<\epsilon$ for some small enough $\epsilon$ if $n_1,...,n_{L-1}>n^*$. Gradient descent in deep and wide non-linear networks converges to the minimum norm solution, i.e. 
\begin{align}
\theta_t\rightarrow  {\arg\min}_{\hat y_0 + \nabla_\theta \hat y_0(\theta-\theta_0) = Y}||\theta-\theta_0||_2. \end{align}
\end{lemma}
\begin{proof}
Note that, as $t\rightarrow\infty$, using the linear differential equation for the function evolution we obtain from \eqref{eq:sols1},
\begin{align}
\theta_t \rightarrow \theta_0 - \nabla_\theta\hat y_0(X)^T(\nabla_\theta\hat y_0(X)\nabla_\theta\hat y_0(X)^T)^{-1}(\hat  y_0(X)-Y).
\end{align}
By simple linear algebra this solution is equivalent to the minimum $L_2$-norm solution of the linear regression problem in \eqref{eq:linappgd}.
\end{proof}

In other words, the weights that the network converges to when trained with gradient descent are such that their distance to the initial weights is the smallest among all weights that satisfy $\lim_{t\rightarrow\infty} \hat y_t=Y$. 
Remark that also for other network architectures the solution found by gradient descent would be the minimum-norm solution for that particular parametrization. Similar to the work in \cite{gunasekar18} it is then of interest to understand the implicit bias in the predictor space. The minimum-norm property could give intuition into why neural networks trained with gradient descent can regularize well in cases where the noise in the data is minimal.  
Since the solution with $L_2$-norm fits the training data with zero error, if significant amount of noise is present in the target points $y_1,...,y_N$, the solution will be overly complex. In order to understand generalization we require a metric that measures the complexity, or smoothness, of the output as a function of the input data. 

As mentioned in Section \ref{sec31} a metric that is used in deep neural networks to study the generalization and robustness to input noise is the input Jacobian. The Jacobian measures the smoothness of the output function with respect to input perturbations and its size correlates well with the generalization capabilities \cite{novak18}. Conveniently, using the analytic expressions in \eqref{eq:sols2} and the fact that the output is thus given by a Gaussian process, we can obtain the model output and its sensitivity at any point $x^*$. We have, using the convergence to the kernel $\Theta_0$ as $n_1,...,n_{L-1}\rightarrow\infty$,
\begin{align}\label{eq:newpointpred}
\hat y_t(x^*) = \hat y_0(x^*) - \Theta_0(x^*,X)\Theta_0^{-1}(I-e^{-\eta \Theta_0 t})(\hat y_0(X)-Y).
\end{align}
Note that this differs from the output if the model were trained in a fully Bayesian way due to $\Theta_0$ not being equal to the kernel $k$. The Jacobian with respect to $x^*$ is given by,
\begin{align}\label{eq:newpointjac}
J_x = \hat y_0'(x^*) - \Theta_0'(x^*,X)\Theta_0^{-1}(I-e^{-\eta \Theta_0 t})(\hat y_0(X)-Y).
\end{align}
The two hyperparameters related to the optimization algorithm, $\eta$ and $t$, influence the size of this Jacobian, i.e. if $t$ and $\eta$ are small, the input Jacobian is small resulting in a smoother, more robust solution. 

\subsubsection{The regularization effects of gradient descent}\label{sec322}
Consider gradient descent over the mean squared error (MSE) loss with a regularization term, i.e. 
\begin{align}
\mathcal{L}(X,\hat y_t,Y)=\frac{1}{2N}||\hat y(X)-Y||_2^2 + \frac{\lambda}{2} ||\theta^L-\theta_0||_2^2,
\end{align} 
In this section we want to understand when and how gradient descent results in a smoothed and regularized solution. 
Applying GD to the loss function we obtain,
\begin{align}\label{eq:regtr}
&\theta_{t+1}^L = \theta_t^L -\frac{\eta}{N}(\nabla_\theta\hat y_t)^T(\hat y_t-Y)-\frac{\eta}{N}\lambda (\theta_t-\theta_0).
\end{align}
For the continuous approximation for the evolution of $\hat y_t$ we obtain,
\begin{align}
&\partial_t \hat y_t =  - \frac{\eta}{N}\Theta_0(\hat y_t-Y)-\frac{\eta}{N}\lambda (\hat y_t - \hat y_0).
\end{align}
We have used informally that when training the network under this loss function, the convergence of the NTK should still hold. 
Solving the continuous forms of these expressions for $\hat y_t(X)$ we obtain, 
\begin{align}\label{eq:solreg}
\hat y_t(X) =
e^{-\frac{\eta}{N}(\Theta_0+\lambda)t}\hat y_0(X)+ (\Theta_0Y+\lambda \hat y_0)\left(\Theta_0+\lambda\right)^{-1}\left(I-e^{-\frac{\eta}{N}(\Theta_0+\lambda)t}\right).
\end{align}
As $t\rightarrow\infty$ the expression obtained for $\hat y(X)$ is similar to the posterior of a Gaussian process where the likelihood is Gaussian with variance $\lambda$, i.e. assuming that there is $\lambda$ noise over the observations in $Y$. As we can observe from the above expression, the convergence is slowed down by the regularization coefficient $\lambda$, so that early stopping leads to smoother solutions than the ones without regularization. At the same time, as $t\rightarrow\infty$, the solution does not converge to $Y$, but to a solution with more smoothness (as observed by a smaller Jacobian when $\lambda$ increases). In case the labels $Y$ contain a significant amount of noise, full convergence is not desirable. If the network is fully converged the network output can be an overly complex function with a poor generalization ability. Regularization prevents this from happening, resulting in a smoother function which is consequently more robust to input perturbations. 

While this result is intuitive it gives insight into the effects of regularization on the function evolution and could form the basis of understanding which optimization algorithm implicitly assumes a particular amount of noise. Note that training with a regularization term is similar to adding $\mathcal{N}(0,\sqrt{\lambda})$-noise to the inputs $X$ during training \cite{reed92}. Therefore, the regularization effect occurs when noise is added to a non-linearly transformed variable, here the input. 

\subsubsection{Gradient descent with stochasticity}\label{sec323}
As has been mentioned in previous work, the noise in SGD can benefit the generalization capabilities of neural networks. As observed in e.g. \cite{kenton18}, \cite{borovykh19}, \cite{kenton17}, \cite{smith18}, \cite{smith19}, \cite{chaudhari18}, \cite{park19} a relationship exists between the test error and the learning rate and batch size, both of which control the variance of the noise, used in the SGD updating scheme. In this section we analyze theoretically where this improvement comes from in order to quantify the effects of noise in a network trained in the \emph{lazy regime}, i.e. when the first-order approximation is sufficient. We analyze the behavior of the network under noisy training by solving explicitly for the output evolution as a function of time $t$, so that, unlike in derivations where the invariant or stationary distribution is used (e.g. \cite{mandt17}), our derivations also give insight into the finite-time behavior of the network output. 

Consider the continuous time approximation of the training dynamics in \eqref{eq:sgdsde}. The evolution of the output function, by It\^o's lemma is then given by, 
\begin{align}
d\hat y_t = \left(-\frac{\eta}{N}\nabla_\theta\hat y_t(\nabla_\theta\hat y_t)^T(\hat y_t-Y)+\frac{1}{2}\sigma^2\frac{\eta^2}{M} Tr(\Delta_\theta\hat y_t(x))_{x\in\mathcal{X}}\right)dt + \frac{\eta}{\sqrt{M}}\sigma \nabla_\theta\hat y_t dW_t.
\end{align}
Under certain assumptions on the SGD training dynamics, the neural network output can be approximated by its first-order approximation,
\begin{align}\label{eq:linappsgd}
\hat y_t \approx \hat y^{lin}_t = \hat y_0 + \nabla_\theta\hat y_0(\theta_t-\theta_0).
\end{align}
 Informally, this holds if the evolution of the original network under SGD does not deviate from the evolution of the linearized network under SGD, which in turn holds if the noise and/or the Hessian are/is sufficiently small. This in turn, by arguments similar to Appendix F in \cite{lee19}, holds if 
 \begin{align}\label{eq:convsgd}
 \sup_{t\in[0,T]} \bigg|\bigg|\Theta_t+\frac{1}{2}\sigma^2\frac{\eta^2}{M} Tr(\Delta_\theta\hat y_t(x))_{x\in\mathcal{X}}-\Theta_0\bigg|\bigg|_{op}\rightarrow 0.
 \end{align}
Assuming that this convergence holds we obtain,
\begin{align}
d\hat y_t = -\frac{\eta}{N}\nabla_\theta \hat y_0(\nabla_{\theta} \hat y_0)^T (\hat y_t - Y)dt +\frac{\eta}{\sqrt{M}} \sigma \nabla_\theta \hat y_0 dW_t,
\end{align}
where we have used that for the linear model approximation $\Delta_\theta \hat y_t=0$.
We obtain, using the limiting behavior of the kernels, 
\begin{align}\label{eq:solsgd}
\hat y_t(X) = \left(I-e^{-\frac{\eta}{N}\Theta_0 t}\right)Y+e^{-\frac{\eta}{N} \Theta_0t}\hat y_0(X)-\sigma\frac{\eta}{N} \int_0^te^{-\frac{\eta}{N} \Theta_0(t-s)}\Theta_0dW_s.
\end{align}
We can study the error between the network output and the true labels in a similar way as in Lemma \ref{lem1}. We have the following result,
\begin{lemma}[Expected MSE for noisy training]\label{lem2}
Consider $n_1,...,n_{L-1}\rightarrow\infty$, so that the deep neural network evolution is governed by the NTK $\Theta_0$. It holds that,
\begin{align}\label{eq:expmse}
\mathbb{E}\left[||\hat y_t - Y||_2^2\right] = \mathbb{E}\left[||e^{-\frac{\eta}{N}\Theta_0 t}(\hat y_0-Y)||_2^2\right] + \sigma^2\frac{\eta^2}{N^2} \int_0^t \sum_{i=1}^N\sum_{j=1}^N \left(\left[e^{-\frac{\eta}{N}\Theta_0(t-s)} \Theta_0\right]_{ij}\right)^2dt.
\end{align}
\end{lemma} 
\begin{proof}
From \eqref{eq:solsgd} we obtain,
\begin{align}
\hat y_t - Y = e^{-\frac{\eta}{N}\Theta_0 t} (y_0-Y) -\int_0^te^{-\frac{\eta}{N}\Theta_0(t-s)}\sigma\frac{\eta}{N} \Theta_0dW_s.
\end{align} 
Using the multi-dimensional It\^o Isometry and 
using the fact that $\hat y_0$ and $W_t$ are uncorrelated and $\mathbb{E}[\hat y_0] = 0 $ and that the expectation of an It\^o integral is zero we obtain the statement.
\end{proof}
Observe that in the one-dimensional case $N=1$ as $t\rightarrow\infty$, we have
\begin{align}
\mathbb{E}\left[||\hat y_t - Y||_2^2\right]\rightarrow \frac{1}{2}\sigma^2 \eta\Theta_0.
\end{align}
From Lemma \ref{lem2} we observe that the stochasticity in noisy gradient descent can result in slower convergence and even in the limit $t\rightarrow\infty$ the MSE may not fully converge on the train data. 

Consider the weight Hessian $H^\theta(X):=\Delta_\theta \mathcal{L}(X, \hat y_t, Y)$ as a metric of smoothness in weight space. We observe that if the network output can be approximated by the linear model in \eqref{eq:linappgd}, 
\begin{align}\label{eq:weighthesslazy}
\mathbb{E}\left[H^\theta(X)\right] =\frac{1}{N} (\nabla_\theta\hat y_0)^T  \nabla_\theta \hat y_0.
\end{align}
The optimization method with noise added to the weight updates -- with the noise given by a scaled Brownian motion -- thus does not add a direct regularization effect, compared to e.g. training with regularization in \eqref{eq:regtr}. In particular, the expected value of the weight Hessian is not affected by the noise coefficient $\sigma$. While the noise thus does not regularize or smooth the network output directly, as happens when using e.g. gradient descent with regularization (see \eqref{eq:solreg}), noisy training in the lazy regime can result in a higher MSE on the train set due to the noise keeping the output from fully converging on the train data. This in turn might lead to a smaller generalization error. 
We remark, that in general for a larger number of training iterations or for a larger noise variance the convergence in \eqref{eq:convsgd} does not hold. Thus under SGD training with a sufficient amount of noise in the weight updates, the model is not given by a linear model and noise does have an effect on the weight Hessian. This is observed in e.g. \cite{smith19}, where even under NTK scaling with SGD training the authors observe better generalization error. The reason this occurs is due to the fact that the model is no longer equivalent to a linear model, in which case -- as we will see in the next section -- noise can decrease the weight Hessian and thus has a regularizing effect on the network output. 

\subsection{Higher-order approximations}\label{sec33}
We are interested in understanding the effects on the output, and in particular the generalization performance, of the noise variance $\sigma^2$ on the $N$-th order approximation of the network output. In the Section \ref{sec24} we were able to obtain an analytic expression for the $N$-th order approximation of the network output. By the properties of the Taylor expansion, the remainder of the $N$-th order Taylor approximation, if the derivatives of the output are bounded, is a function of $(\theta_t-\bar\theta)^{N+1}$. The approximation is thus most accurate around $\bar\theta$. If we set $\bar\theta=\theta$, with $\theta$ the value of the weight at initialization, then in a small interval around initialization at $t_0$ the linear approximation might be sufficient; as training progresses, we may require higher order terms in order to approximate the network output. We remark here that a Taylor series need not be convergent; nevertheless in an interval around initialization the approximation is accurate enough (see Remark \ref{remarkconv}). 

In Section \ref{sec2} we showed that if the network output is \emph{linear} in the weights, noise does not have an explicit effect on the function smoothness, as observed from the weight Hessian in \eqref{eq:weighthesslazy} being unaffected by the noise. From Theorem \ref{cor1} we observe that unlike in the first-order approximation of $\hat y_t(X)$, for higher order approximations, i.e. a setting in which the output is \emph{non-linear} in the weights, the noise has an effect on the expected value. In order to quantify generalization through a heuristic metric, we will again consider the weight Hessian $H^\theta$ of the loss function of the approximated output function $\hat y_t^{(N)}$. A low sensitivity to weight perturbations, i.e. a smaller Hessian, would imply better generalization.
Observe that in the one-dimensional weight setting, using the approximated output function $\hat y_t^{(N)}$, we have 
\begin{align}
\mathbb{E}\left[H^\theta(X)\right] =& \frac{1}{N}\sum_{\substack{k\leq N-2,j\leq N:\\
                  k+j=n}} \frac{1}{k!}(\partial_\theta^{k+1}\hat y_t(X)|_{\bar\theta} )^T\left(\frac{1}{j!}\partial_\theta^{j}\hat y_t(X)|_{\bar\theta}-\caratt_{k=n}Y\right)(\theta-\bar\theta)^{k+j}\\
                  &+\frac{1}{N}\sum_{\substack{k\leq N-1,j\leq N-1:\\
                  k+j=n}} \frac{1}{k!}(\partial_\theta^{k+1}\hat y_t(X)|_{\bar\theta} )^T\frac{1}{j!}\partial_\theta^{j}\hat y_t(X)|_{\bar\theta}(\theta-\bar\theta)^{k+j}.
\end{align}
We observe from the above expression that the Hessian depends on $(\theta_t-\bar\theta)^m$, $m=0,1,...,2N-2$ and thus in order to obtain insight into the structure of the Hessian, we require the estimates $u_m^n$ from Theorem \ref{cor1}. 

From the analytic expressions for the weight moments it can be observed that the noise variance $\sigma$ can result in a positive or negative addition, where negativity occurs due to the multiplication with the drift which is negative if the gradient is positive, to the weight Hessian. In particular, under $\bar \theta = \theta$, in the one-dimensional case the general structure of the terms is,
\begin{align}
&\left(-\frac{\eta}{N}\partial_\theta^j\hat y^T(\hat y-Y)\right)^k(t-t_0)^m, \;\;\;\; \left(-\frac{\eta}{N}\partial_\theta^j\hat y^T\partial_\theta^l\hat y\right)^k(t-t_0)^m,\\
&\left(-\frac{\eta}{N}(\partial_\theta^{j}\hat y)^T(\partial_\theta^{l}\hat y-Y)\right)^k\sigma^n (t-t_0)^m,\;\;\;\; \left(-\frac{\eta}{N}(\partial_\theta^{j}\hat y^T\partial_\theta^{l}\hat y\right)^k\sigma^n (t-t_0)^m,
\end{align}
for some $j,k,l,n,m\in\mathbb{N}$. 

While the interplay between noise and flatness is intricate, the above formulas show that for the noise to decrease the value of the Hessian, $k$ needs to be uneven and the noise needs to be scaled such that it is `aligned' with the first- and higher-order derivatives of the loss surface. There seems to exist a balance such that a too large noise can overshoot and result in a too negative or too positive Hessian, while a sufficient amount of noise can decrease the Hessian and result in a regularization of the output function. Furthermore, the effects of the noise variance are dampened by $(t-t_0)$ if this is small. This is related to the above observation that around initialization the linear approximation, on which noise has no effect, may be sufficient. Therefore, noise does not seem to have a significant effect on the output at the beginning of training, and the impact increases as training proceeds and the network output deviates from the linear model. 

From the above formulas, in the one-dimensional setting the size of the noise variance $\sigma$ needs to `match' the gradient of the loss surface. A similar observation can be said to hold in the multi-dimensional case, where in order for the noise to have a regularizing effect, the noise variance in a particular direction needs to be `match' the gradient in that direction. This is similar to the results in \cite{zhu19} where the authors observed that the noise needs to be `aligned' with the derivatives of the loss surface in order to increase generalization capabilities. An isotropic covariance matrix, i.e. one in which the noise is of the same size in all directions, can thus decrease the Hessian in certain directions while increasing it in others. We leave the multi-dimensional study for future work, since the full-scale analysis is beyond the scope of this work. Nevertheless, our analytical expressions are able to provide relevant insights into the effects of the noise on the generalization capabilities. To conclude this section, for network which is \emph{non-linear} in the weights, if the generalization capabilities of the network output can be quantified by the smoothness of the loss surface, this smoothness, as observed from the analytic expressions, is dependent on the noise variance and the derivatives of the output and loss function at initialization. 

\section{Numerical results}\label{sec4}
In this section we present numerical results that validate the theoretical observations made in the previous sections empirically. We consider here the setting of time series forecasting (regression) so that a number of previous time steps is used to predict the next-step-ahead value. Previous work on generalization typically considers the classification setting for which slightly different results may hold. In particular worth mentioning is the following: in classification even after full convergence on the train set the test error might continue to change due to a changing margin. In particular this means that minima with zero train error might still be 'smooth' in terms of their complexity and have a good out-of-sample performance. In the regression setting we typically observe that if a zero regression error is obtained the trained function will be of high complexity and thus result in a large test error. Nevertheless, the effects of noise and other hyperparameters are for a large part similar for both classification and regression. 

In particular we use two time series: a sine function with noise given by $y_i=\sin(0.1t_i)+c\epsilon_i$ with $c=0.3$, $\epsilon\sim\mathcal{N}(0,1)$ and $t_i\in\{0,1,...,100\}$ and the daily minimum temperature in Melbourne, Australia. The sine dataset consists of 100 train and 100 test point, while the weather dataset consists of 600 train and 100 test points. The network uses $k$ historical datapoints $t_{n-k},...,t_n$ to predict $t_{n+1}$, where $k=5$ for the sine function and $k=20$ for the temperature dataset. The network is trained to minimize the MSE loss. Unless otherwise mentioned, the network consists of $L=5$ layers with a ReLU activation and $n_l=200$ nodes per layer \footnote{The number of nodes in the NTK regime should be such that the model is close to the linear model. The deeper the network, the wider the layers should be to achieve NTK convergence. The choice of 200 nodes per layer with five layers seems to be sufficient for the network to be close to the linearized network.}, the number of training iterations is set to $n_{its} = 10.000$, the learning rate is set to $\eta=0.01$ in the non-lazy regime and $\eta=1$ in the lazy regime with $\beta=0.1$. 
The results are presented averaged over 20 minima; these minima are obtained by running the optimization algorithm 20 times from a fixed initialization with random noise; the sine function noise is sampled per minimum. 
\subsection{Approximating the neural network output}
In this section we approximate the neural network output with the Taylor expansion from \eqref{eq:taylor1d}. We train the neural network with full-batch gradient descent. We compute the divergence between the true network output and its first and second order approximation at each point of training. The results are presented for the sine dataset, where we use the output for one train observation to measure the squared divergence  i.e. $\left(\hat y_t(x^1)-\hat y_t^{(N)}(x^1)\right)^2$, for $N=1,2$. We observe from Figures \ref{fig01}-\ref{fig02} that for wide networks under the NTK scaling the divergence is small; in other words, under the NTK scaling for sufficiently wide networks the first order approximation is sufficient. For smaller widths or larger iteration numbers the divergence increases and adding the second-order term can significantly improve the approximation. Similarly for the regular scaling, adding higher order terms improves the quality of the approximation. To conclude, adding higher-order terms can improve the approximation in both the NTK and regular regime. The linear approximation is only accurate in the wide network with NTK scaling and a relatively small number of training iterations; higher order terms are needed in order to approximate the network output for narrow neural networks or for large iteration numbers. 

\begin{figure}[H]
  \centering
  \includegraphics[width=0.35\textwidth]{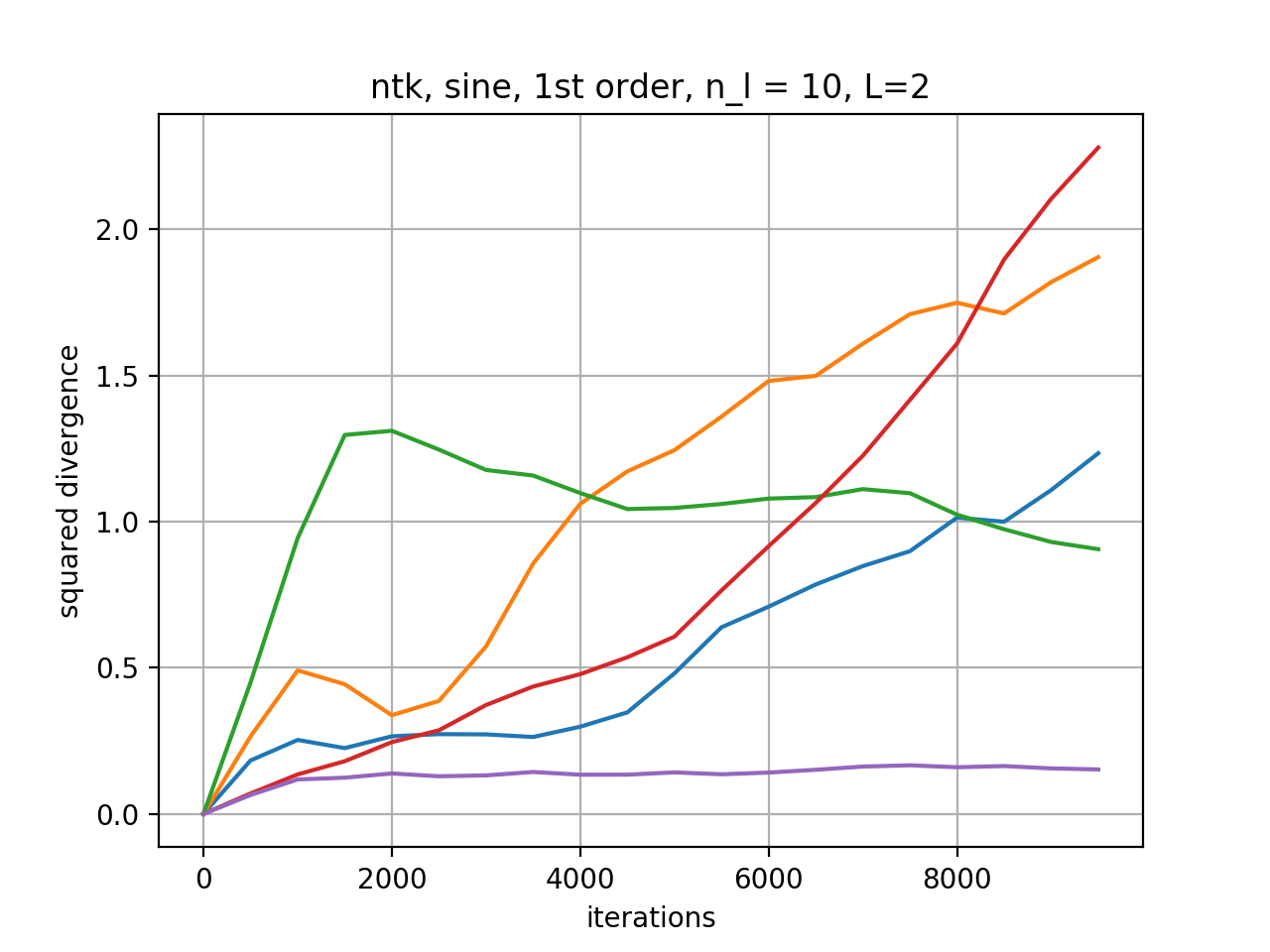} 
    \includegraphics[width=0.35\textwidth]{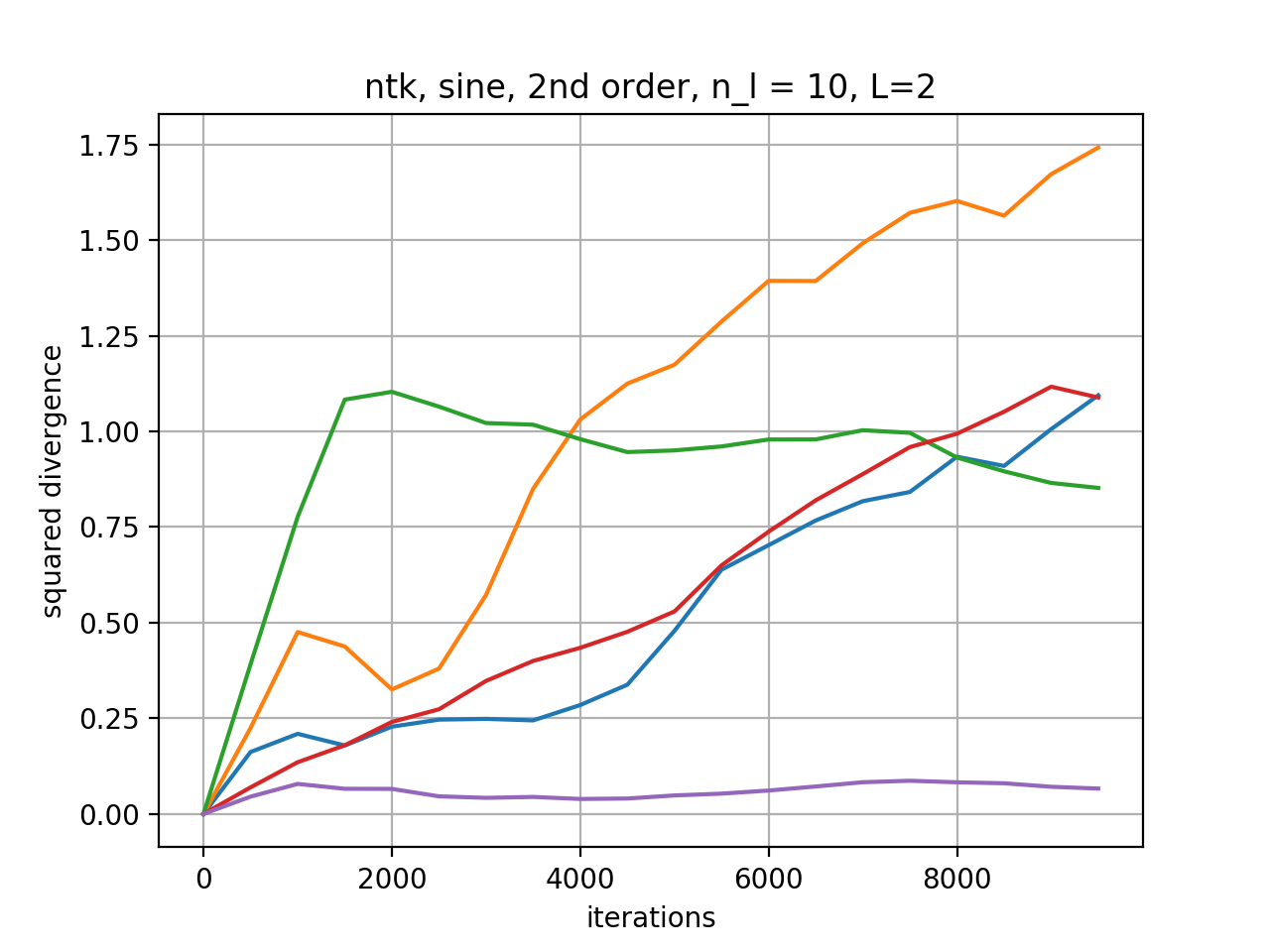} \\
      \includegraphics[width=0.35\textwidth]{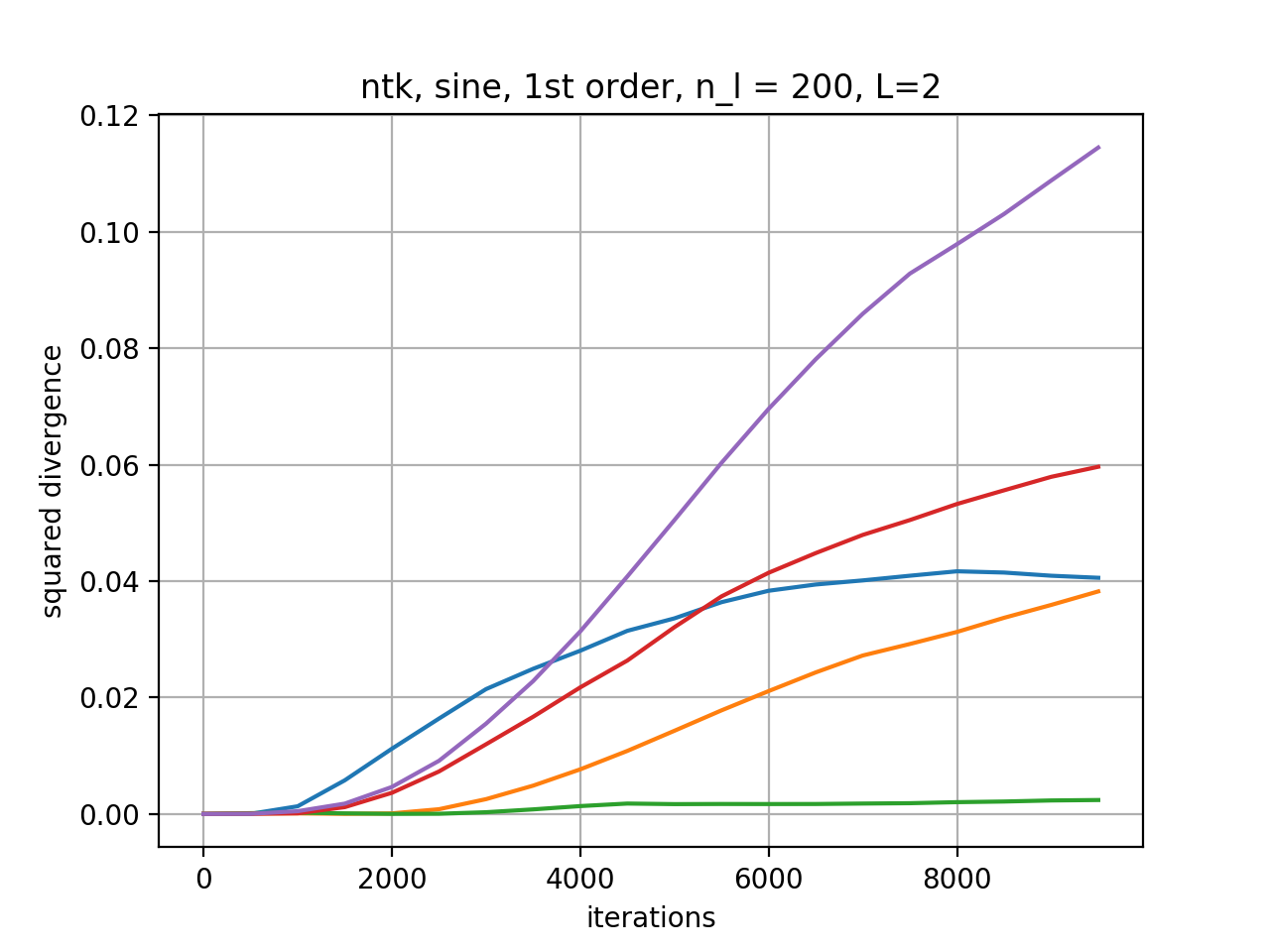}
  \includegraphics[width=0.35\textwidth]{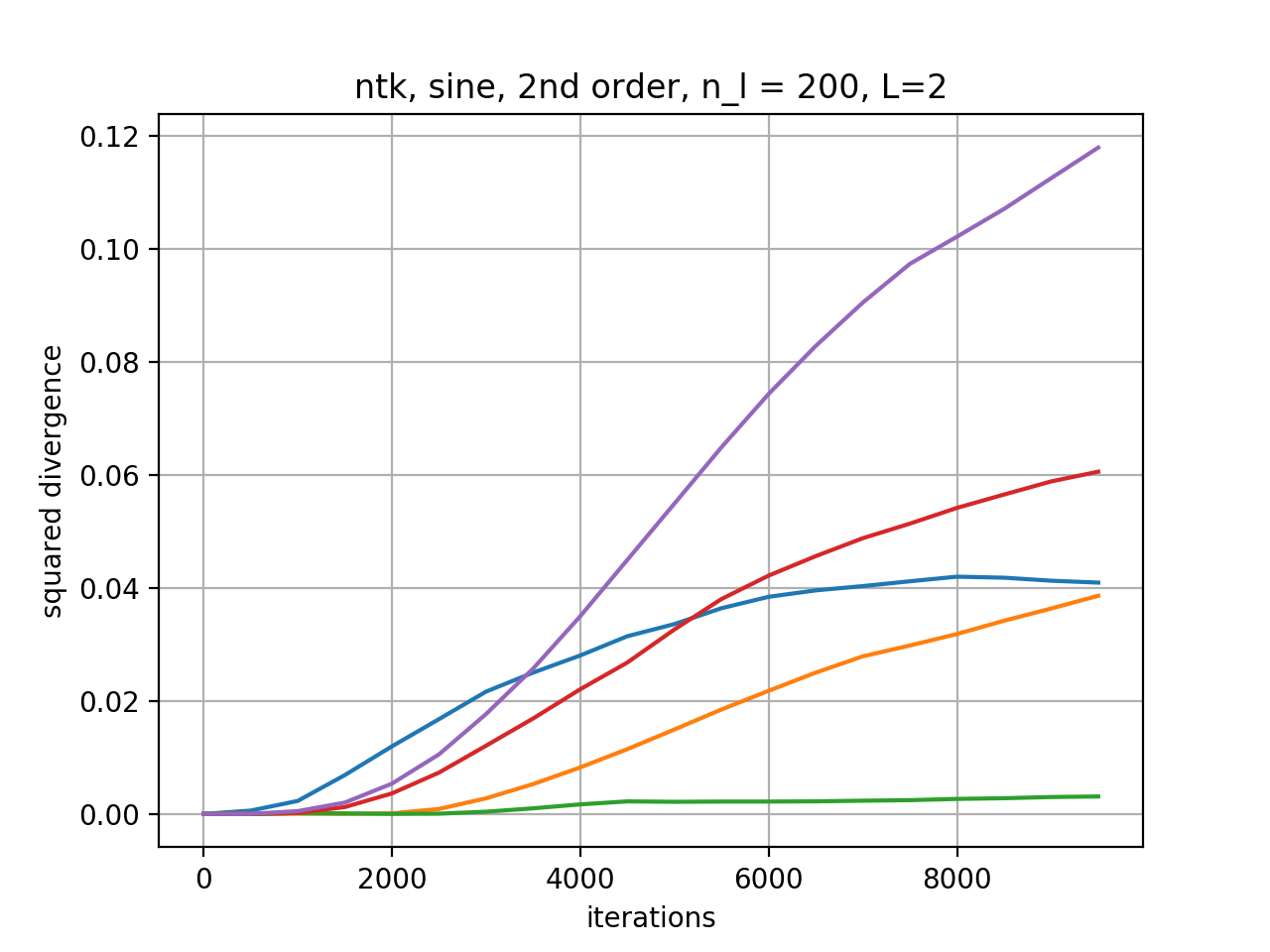}
        \caption{Five simulations of the divergence between the neural network output and the Taylor approximation of the model for NTK scaling. A larger number of iterations results in higher divergence, while increasing layer width decreases divergence; when the divergence is already small adding higher-order Taylor expansion terms does not significantly improve the approximation.}\label{fig01}
\end{figure}

\begin{figure}[H]
  \centering
  \includegraphics[width=0.35\textwidth]{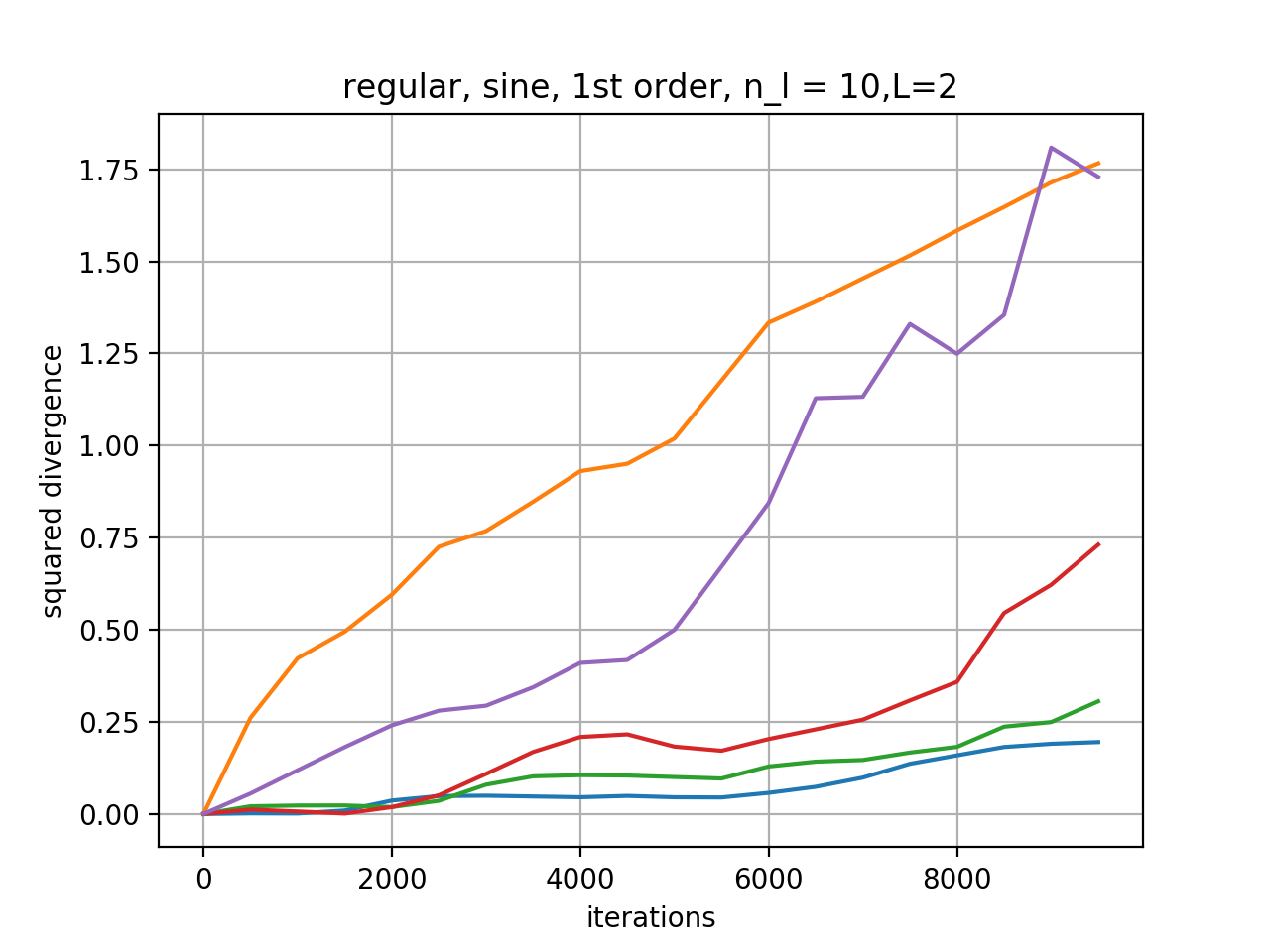} 
    \includegraphics[width=0.35\textwidth]{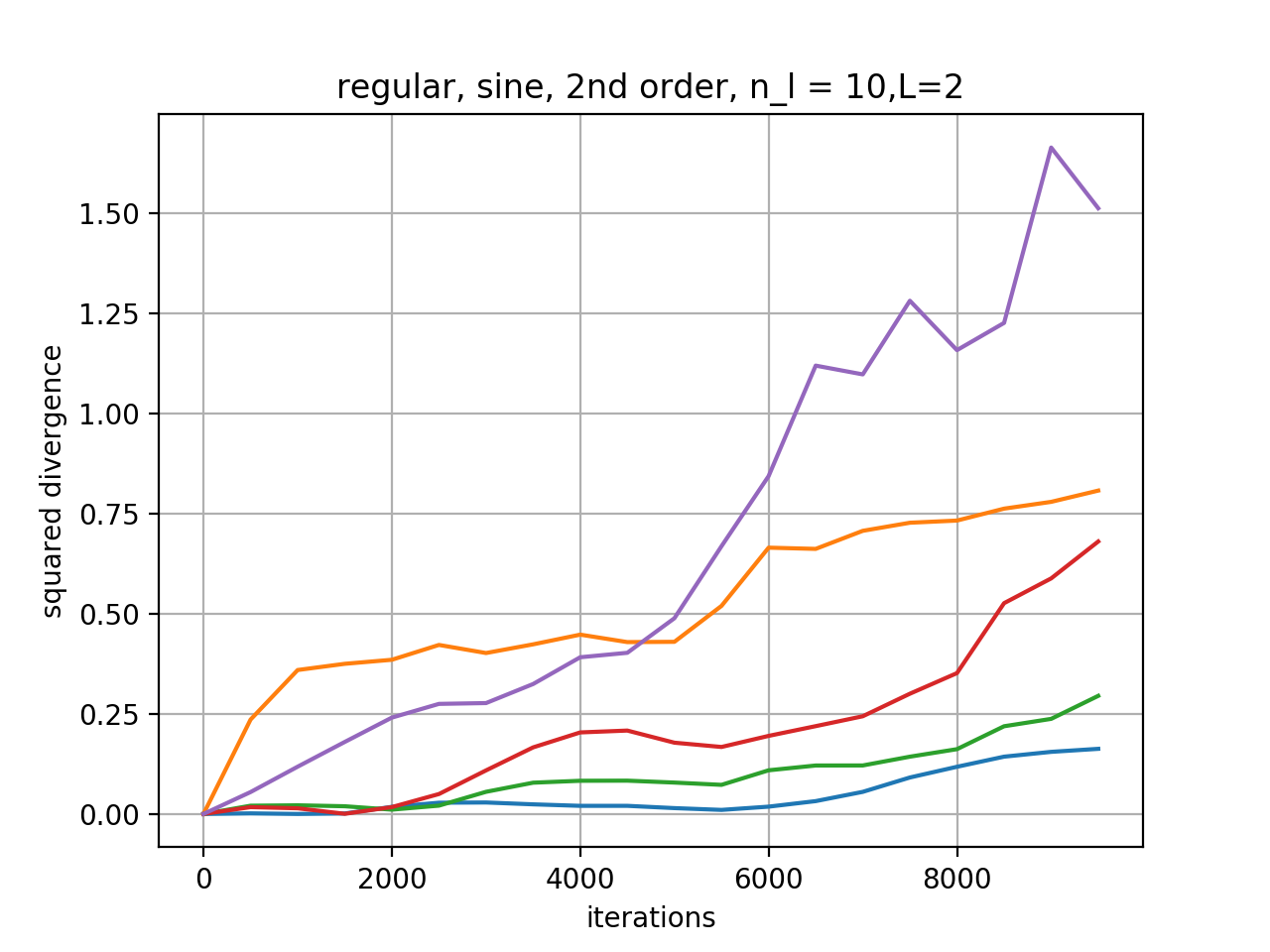} \\
      \includegraphics[width=0.35\textwidth]{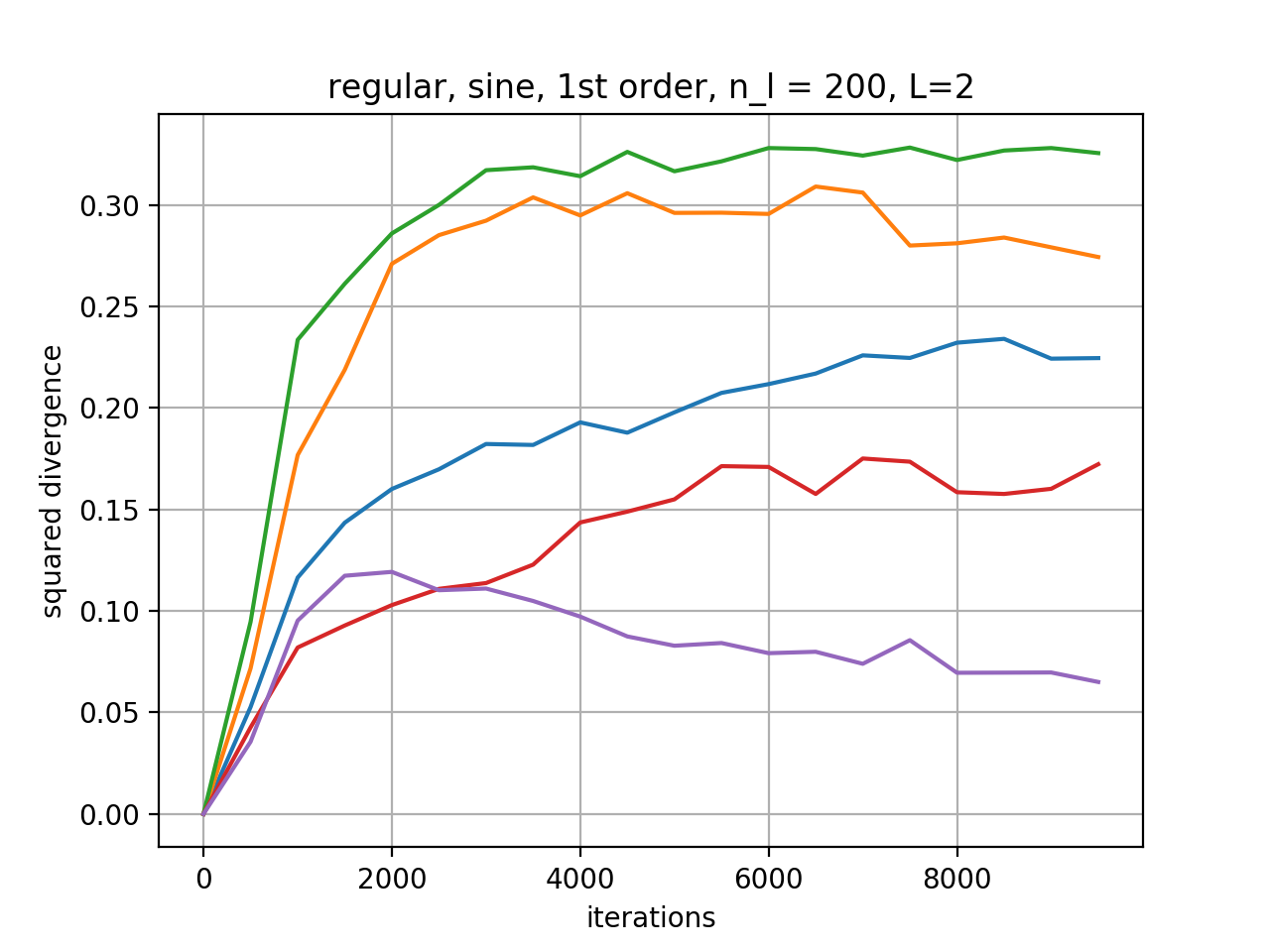}
  \includegraphics[width=0.35\textwidth]{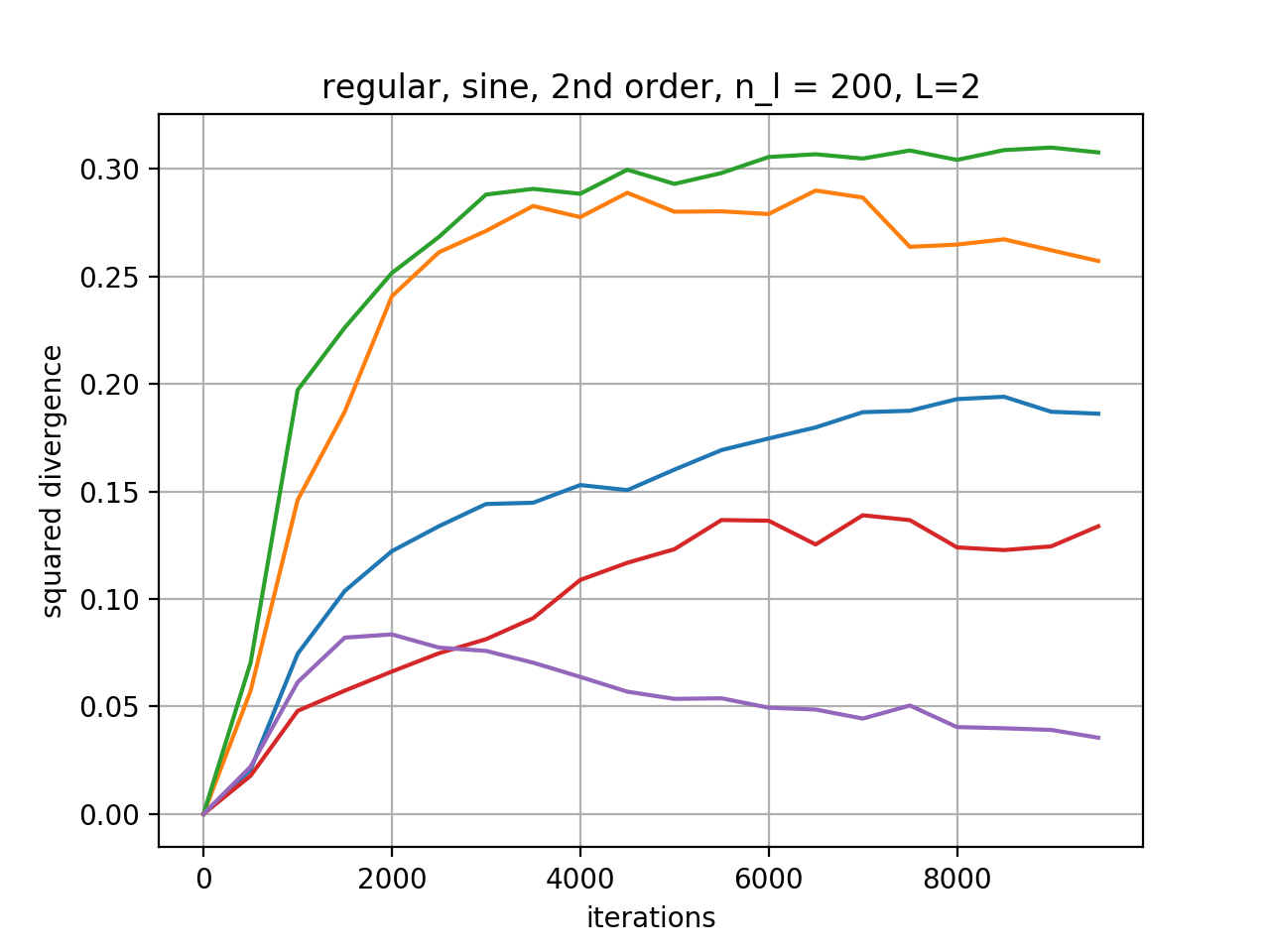}
        \caption{Five simulations of the divergence between the neural network output and the Taylor approximation of the model for regular scaling. The linear model is not an accurate approximation of the dynamics in this regime and adding higher order approximation terms can significantly improve the approximation. }\label{fig02}
\end{figure}

\subsection{Gradient descent optimization}
Here we study the effects of the gradient descent hyperparameters in the lazy training regime. From expression \eqref{eq:newpointpred} and the input Jacobian in \eqref{eq:newpointjac}, besides the structure of the kernel and the data itself, the training hyperparameters influencing the smoothness of the solution are the learning rate $\eta$ and the number of iterations $t$. In Figures \ref{fig1}-\ref{fig2} we show the effects of these hyperparameters. As expected, from Figure \ref{fig3} it can be seen that the function is smoother if $t$ is small. As observed in Figure \ref{fig1}, the convergence is governed by the amount of noise present in the data (i.e. the `roughness' of the function); a slower convergence is obtained for higher noise levels in the noisy sine function. A larger learning rate also results in faster convergence.

\begin{figure}[H]
  \centering
  \includegraphics[width=0.35\textwidth]{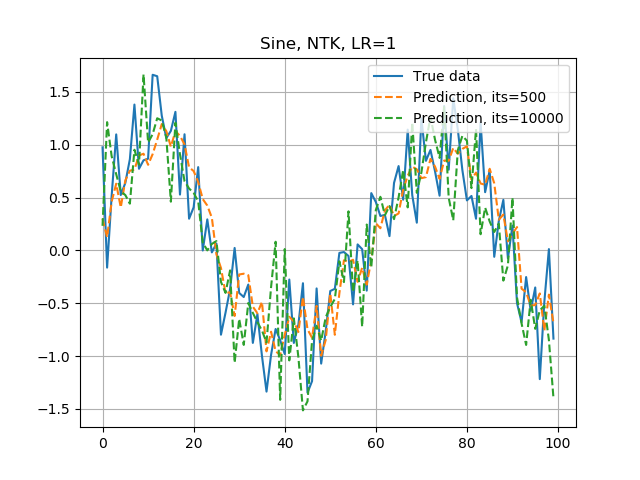}
  \includegraphics[width=0.35\textwidth]{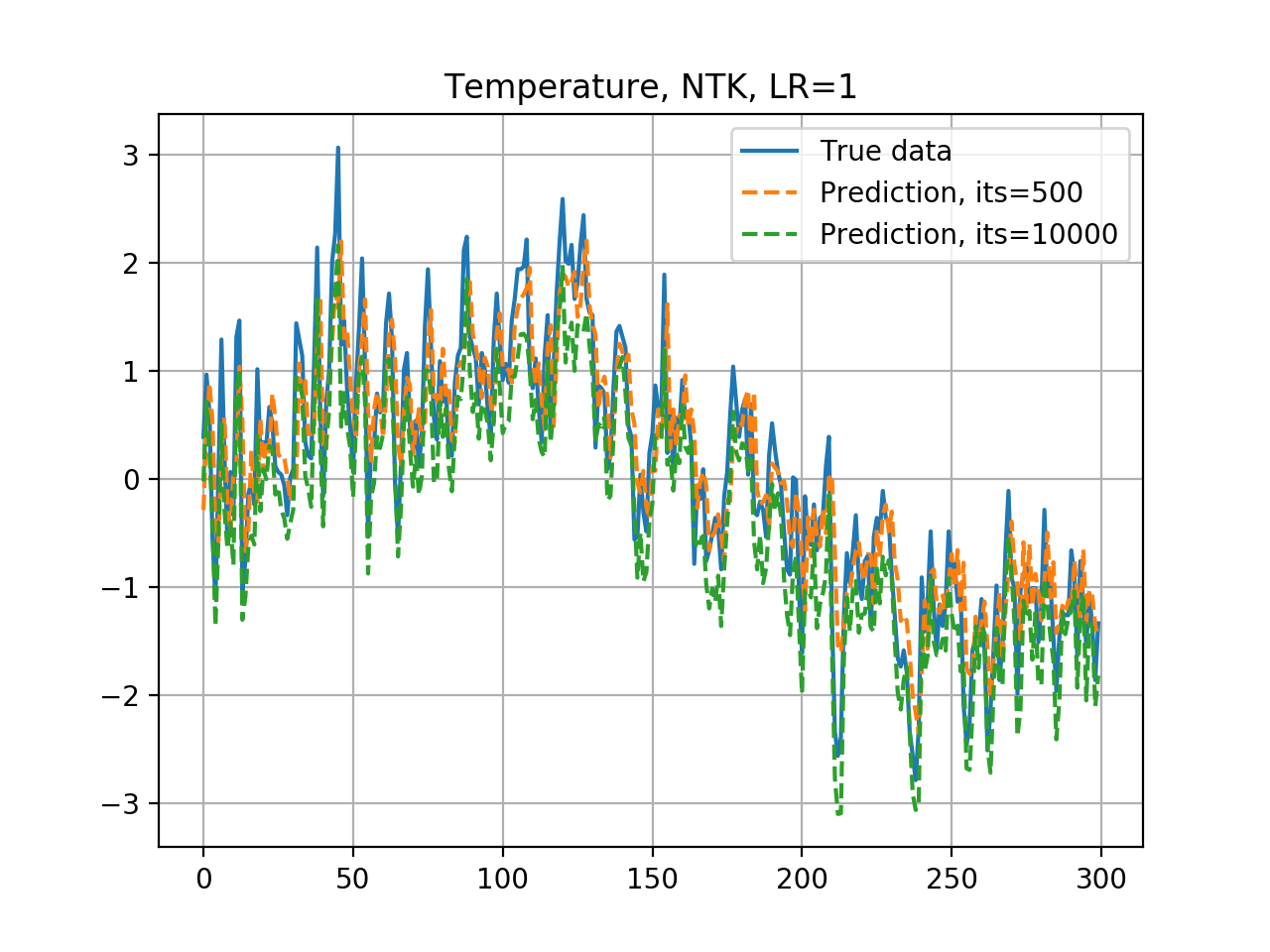}
        \caption{Learned function for different number of training iterations for the noisy sine function (L) and the weather dataset (R). A smoother function is obtained with fewer training iterations.}\label{fig2}
\end{figure}

\begin{figure}[H]
  \centering
  \includegraphics[width=0.35\textwidth]{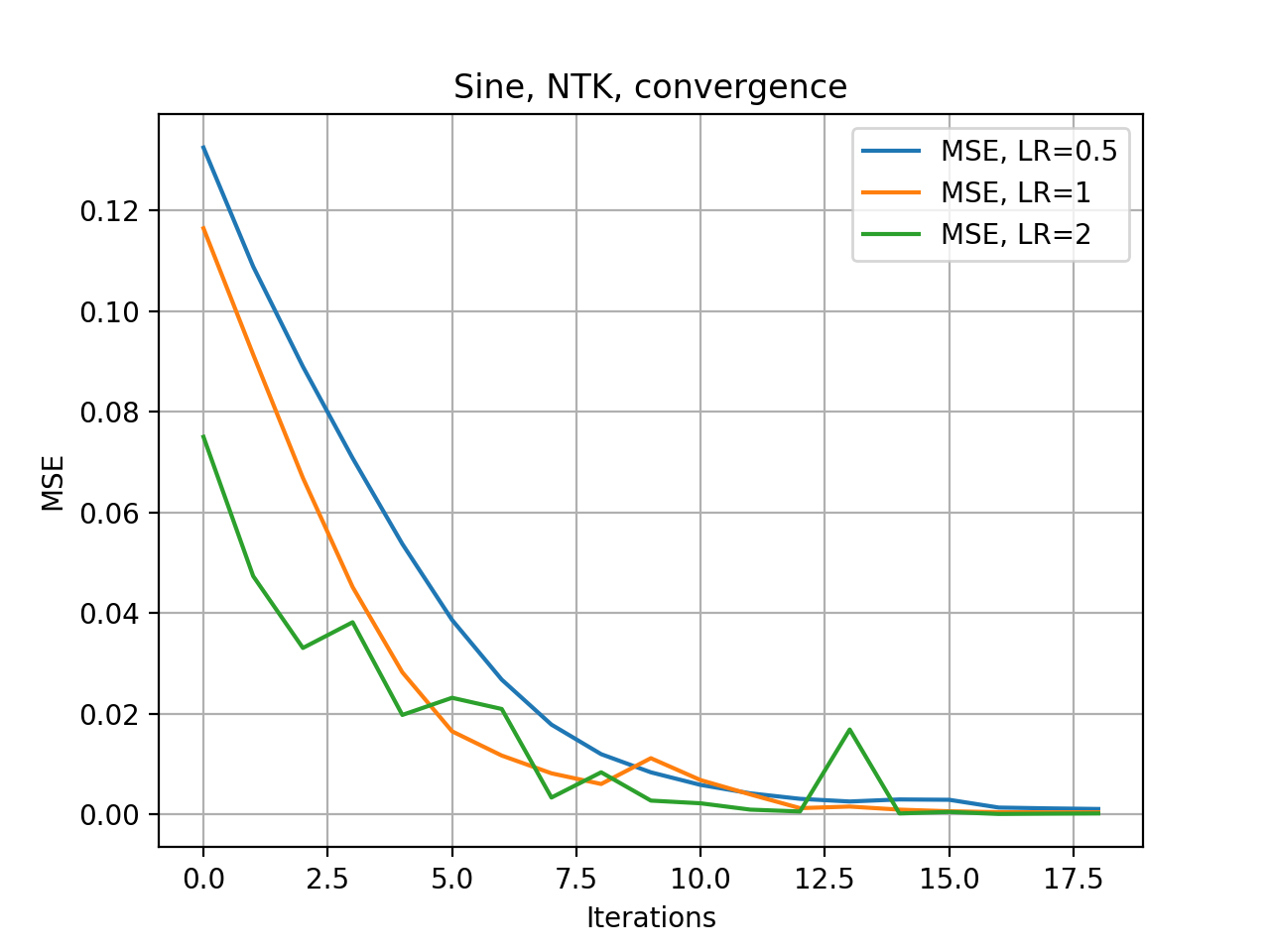} 
  \includegraphics[width=0.35\textwidth]{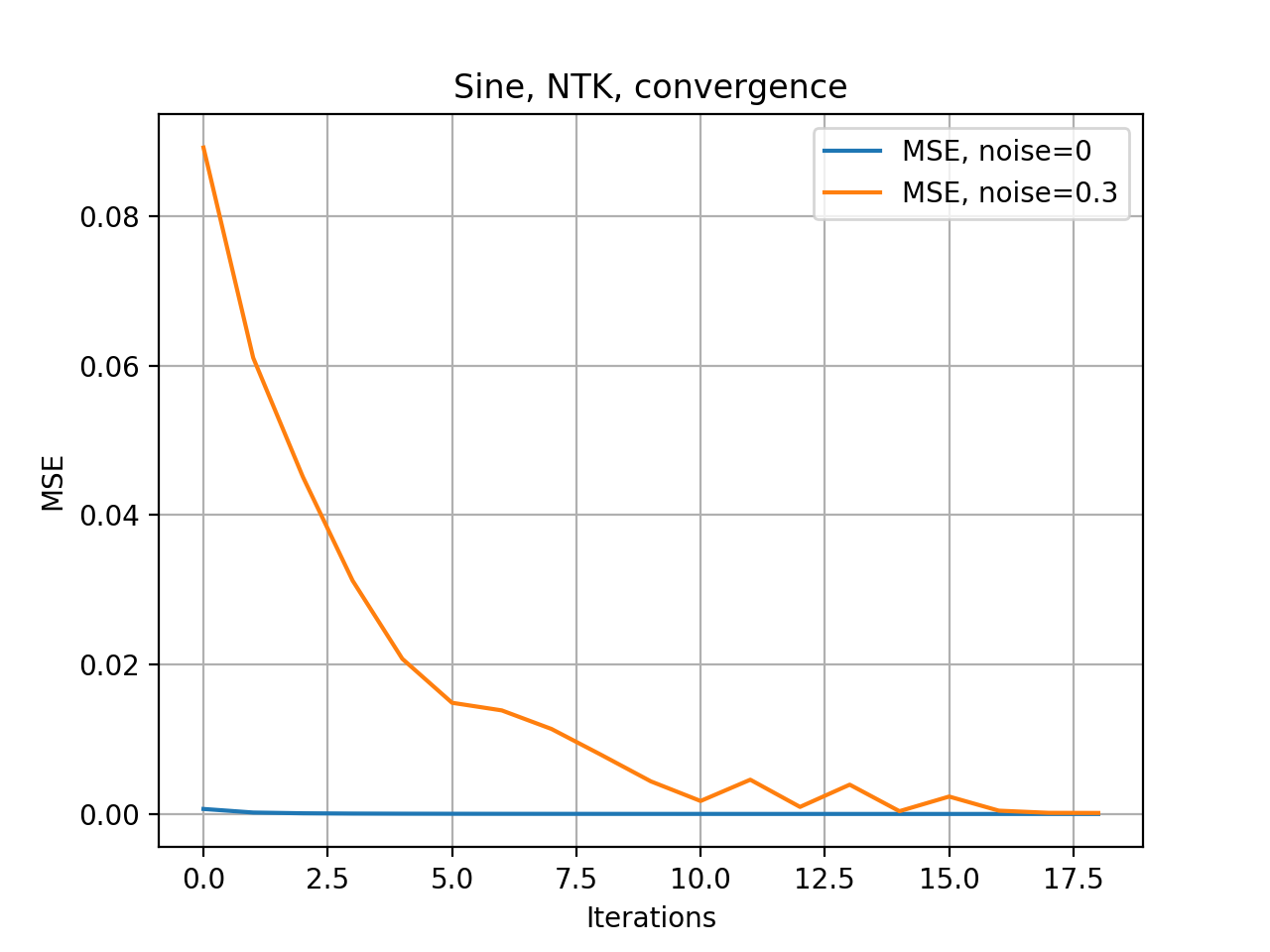}
        \caption{Convergence of the neural network on the noisy sine function for varying learning rates (L) and different noise variance in the data $c$ (R). We observe a faster convergence for larger learning rates and slower convergence for higher noise levels.}\label{fig1}
\end{figure}

\subsection{Noisy training}
In this section we consider noisy optimization algorithms, either by adding explicit noise to the weight updates or through the noise in SGD. We compare the regularization effects of noise in the lazy and non-lazy settings. 
\subsubsection{Lazy regime}
In this section we consider training a neural network under the NTK scaling \eqref{eq:nnntk}. Noise is added directly in function space, so that the following update rule is used: 
\begin{align}\label{eq:fnnoise}
w_{t+1} = w_t -\frac{\eta}{N}(\hat y_t-Y)\nabla_\theta\hat y_t(X) -\frac{\eta}{N}\nabla_\theta\hat y_t(X)\epsilon_t, \;\; \textnormal{with }\epsilon_t\sim \mathcal{N}(0,\sigma^2).
\end{align}
For a small number of iterations with a low noise coefficient under the NTK scaling the linear model is a good approximation of the neural network output (see also the bottom left plot in Figure \ref{fig01}). As observed in Figure \ref{fig3} noise in the case of the linear model affects the train and test error, however does not contribute to a better generalization performance. As seen in Figure \ref{fig31}, and as expected from \eqref{eq:weighthesslazy}, noise has little impact on the trace of the weight Hessian. Since the weight Hessian, which is a metric for generalization performance, is not affected, the out-of-sample performance is not improved by adding more noise, which was observed from Figure \ref{fig3}. This coincides with the results from Section \ref{sec32}. 

\begin{figure}[H]
  \centering
  \includegraphics[width=0.35\textwidth]{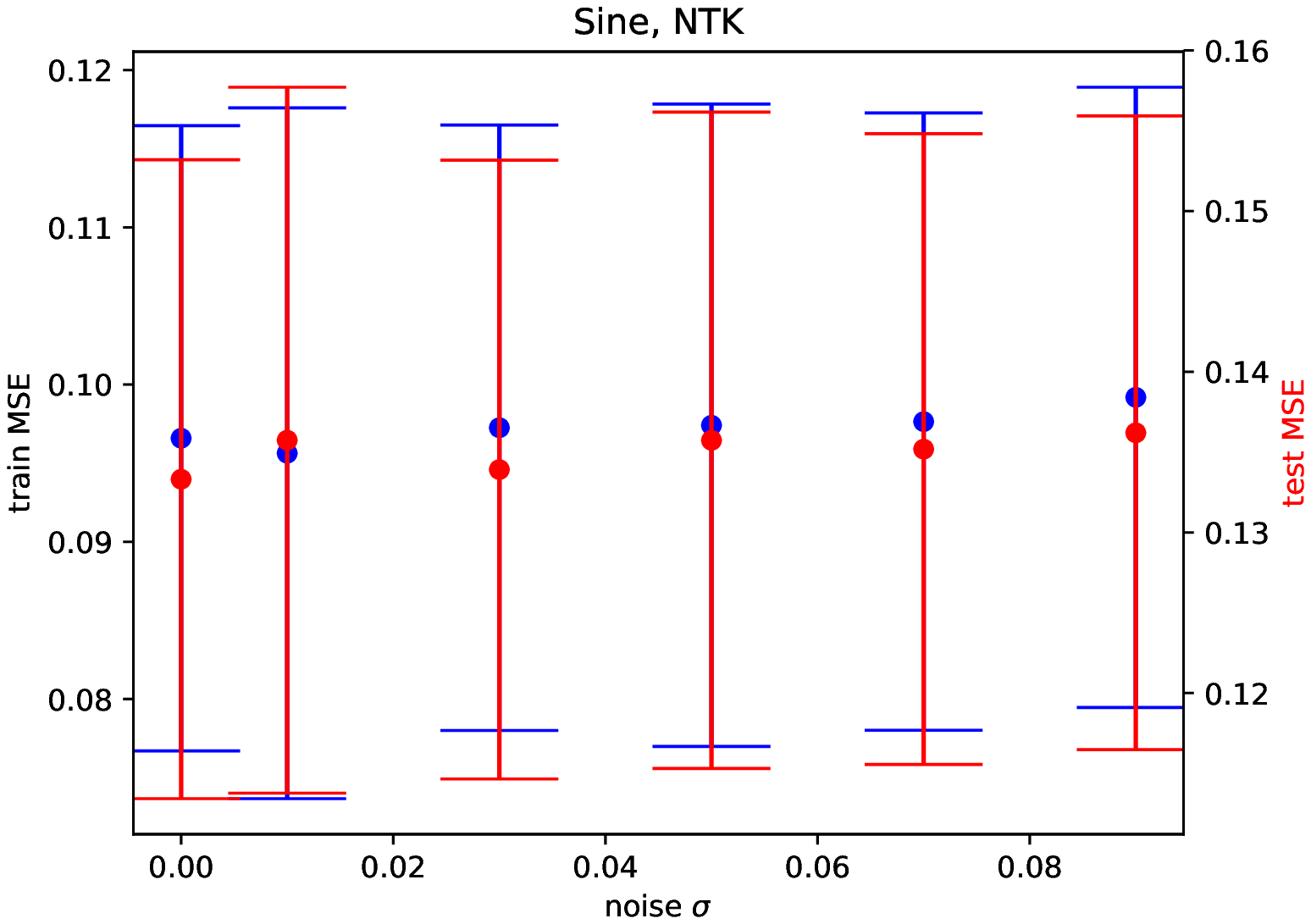}
 \includegraphics[width=0.35\textwidth]{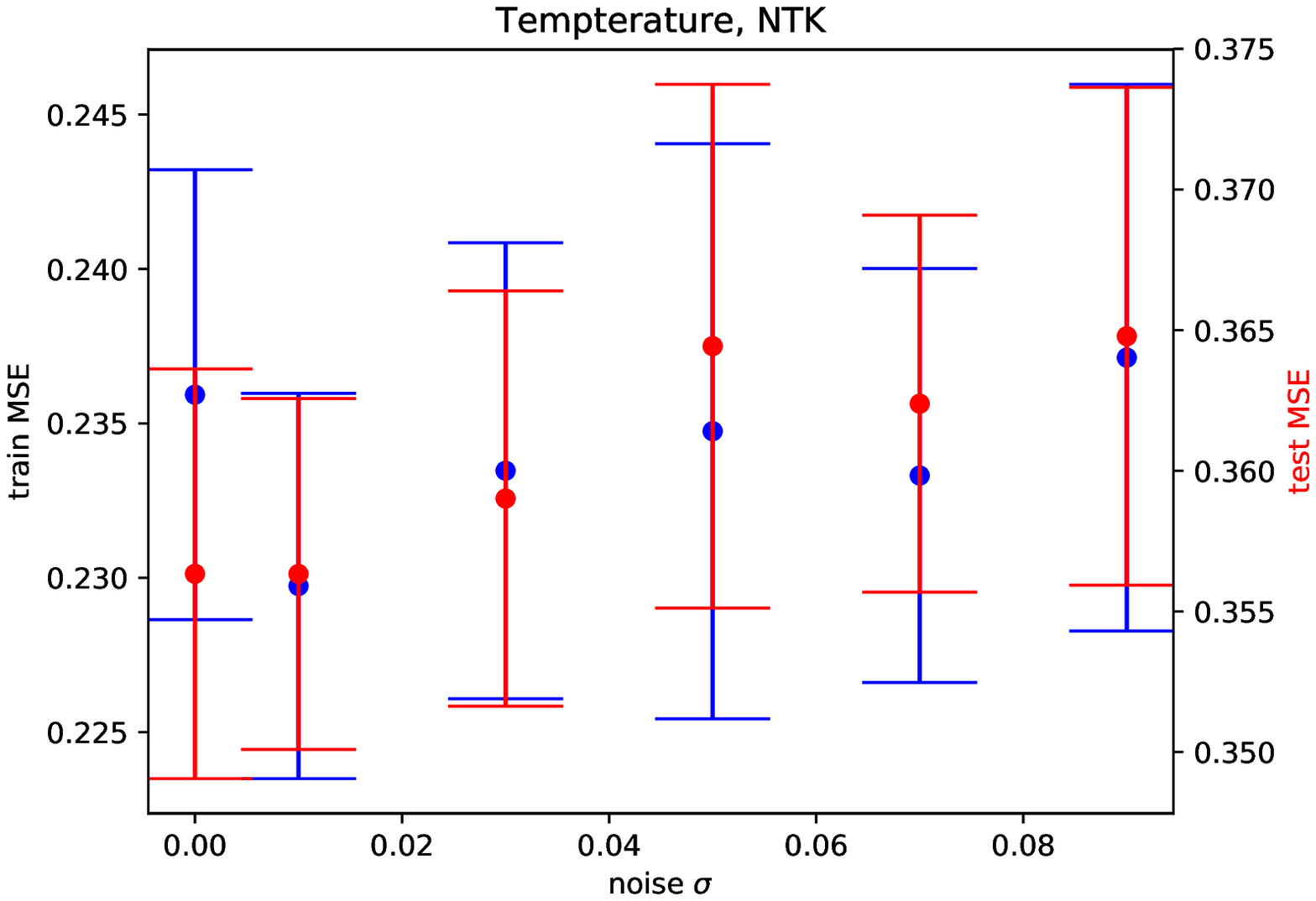}
        \caption{ The effects of noisy training with 500 iterations, with Gaussian noise with standard deviation $\sigma$ added in function space, in the lazy regime on the train and test error on the sine (L) and temperature (R) data. A larger noise variance increase the train error, but, as expected, does not improve the out-of-sample performance. }\label{fig3}
\end{figure}

\begin{figure}[H]
  \centering
   \includegraphics[width=0.35\textwidth]{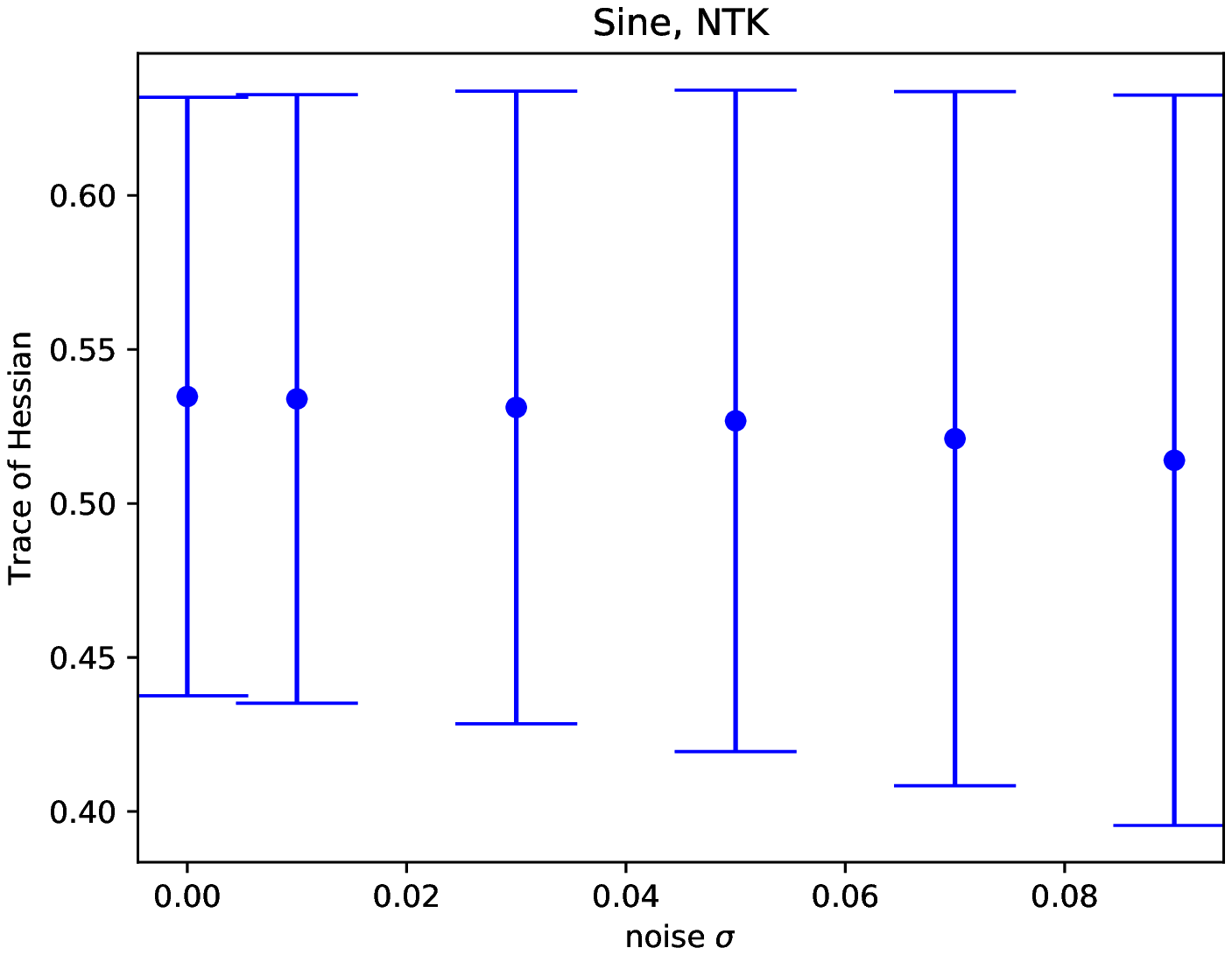}
\includegraphics[width=0.35\textwidth]{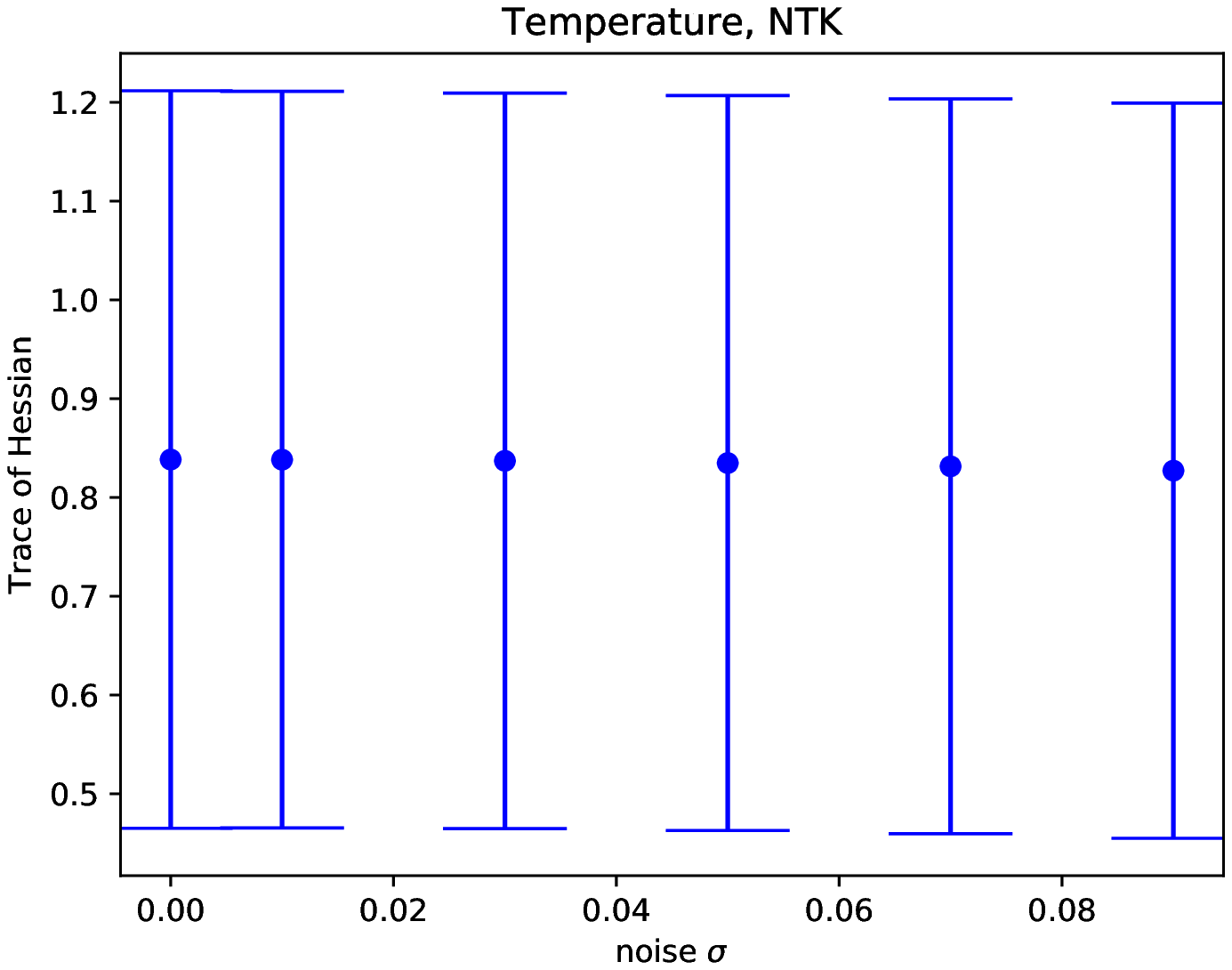}
        \caption{ The effects of noisy training with 500 iterations, with Gaussian noise with standard deviation $\sigma$ added in function space, in the lazy regime on the trace of the first- and last-layer weight Hessian on the sine (L) and temperature (R) data. The trace of the Hessian is not affected.}\label{fig31}
\end{figure}

\subsubsection{Non-lazy regime}
In Figure \ref{fig3a} we show the results of noisy training averaged over 20 runs with the optimization from \eqref{eq:fnnoise} for 10000 iterations while still using the NTK weight scaling. As seen from Figures \ref{fig3a}-\ref{fig3a1}, for a higher number of training iterations, the noise influences the trace of the weight Hessian and results in a slightly better test performance. The observed regularization effect of noise can be explained by the fact that in the finite-width model more noise allows to deviate from the linear model, so that unlike in Figure \ref{fig3} the model is no longer in the lazy training regime and a small regularization effect is observed.  

\begin{figure}[h!]
  \centering
  \includegraphics[width=0.35\textwidth]{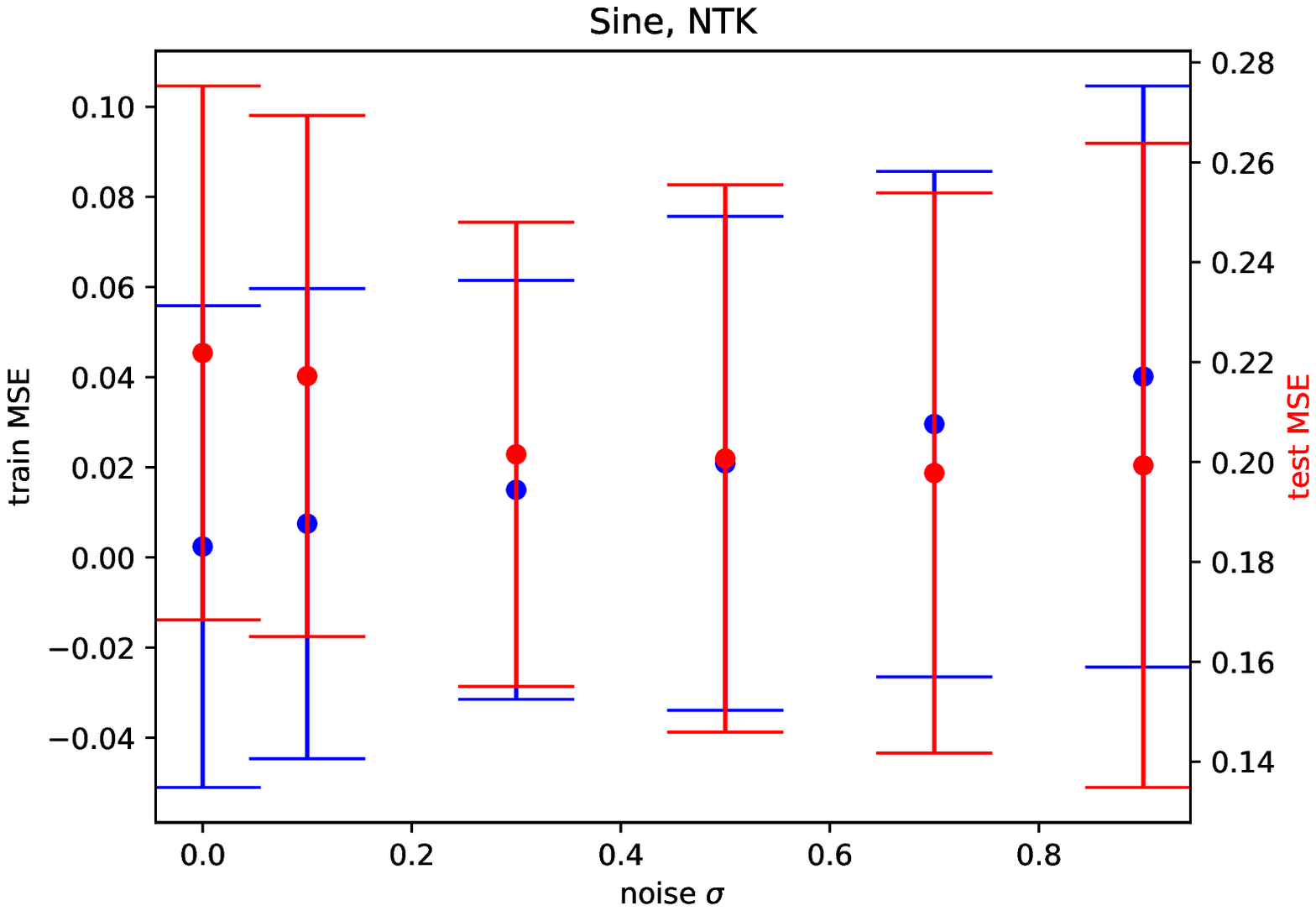}
 \includegraphics[width=0.35\textwidth]{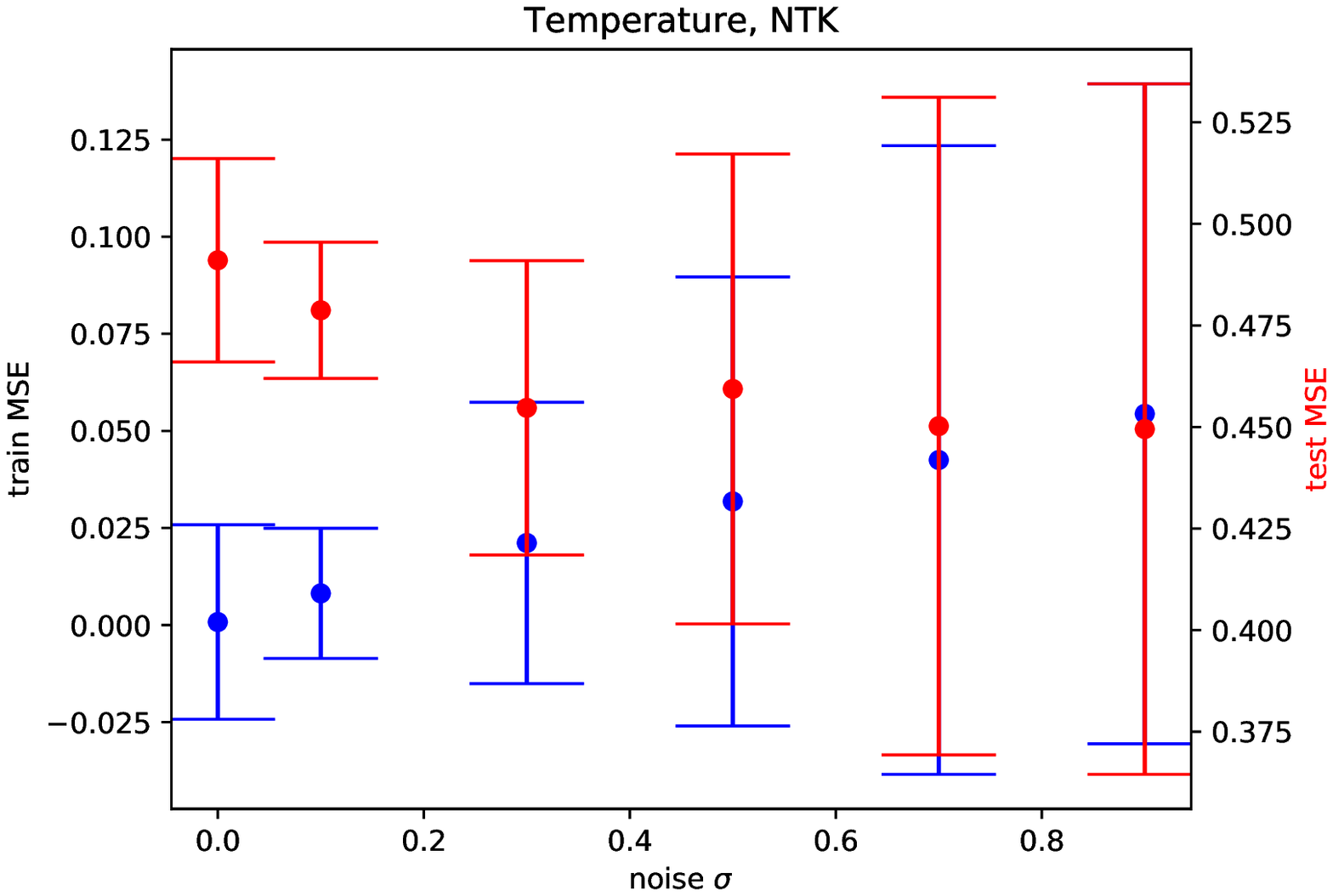}
        \caption{The effects of noisy training with 10000 iterations, with Gaussian noise with standard deviation $\sigma$ added in function space, with NTK scaling on the train and test error on the sine (L) and temperature (R) data. A larger noise variance increases the train error and the generalization error does decrease due to the larger train error. The variance for both the train and test errors increases as more noise is introduced.}\label{fig3a}
\end{figure}

\begin{figure}[h!]
  \centering
   \includegraphics[width=0.35\textwidth]{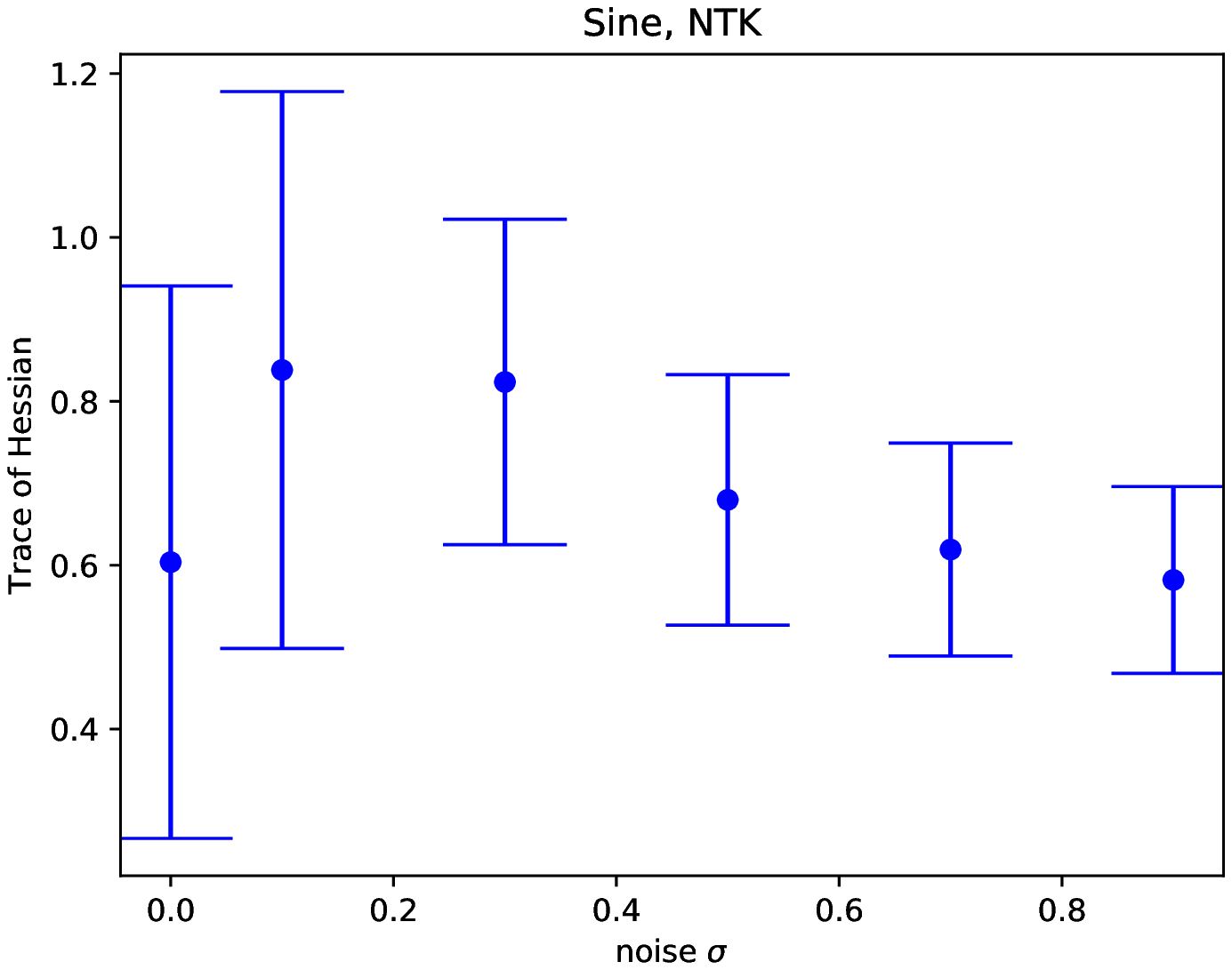}
\includegraphics[width=0.35\textwidth]{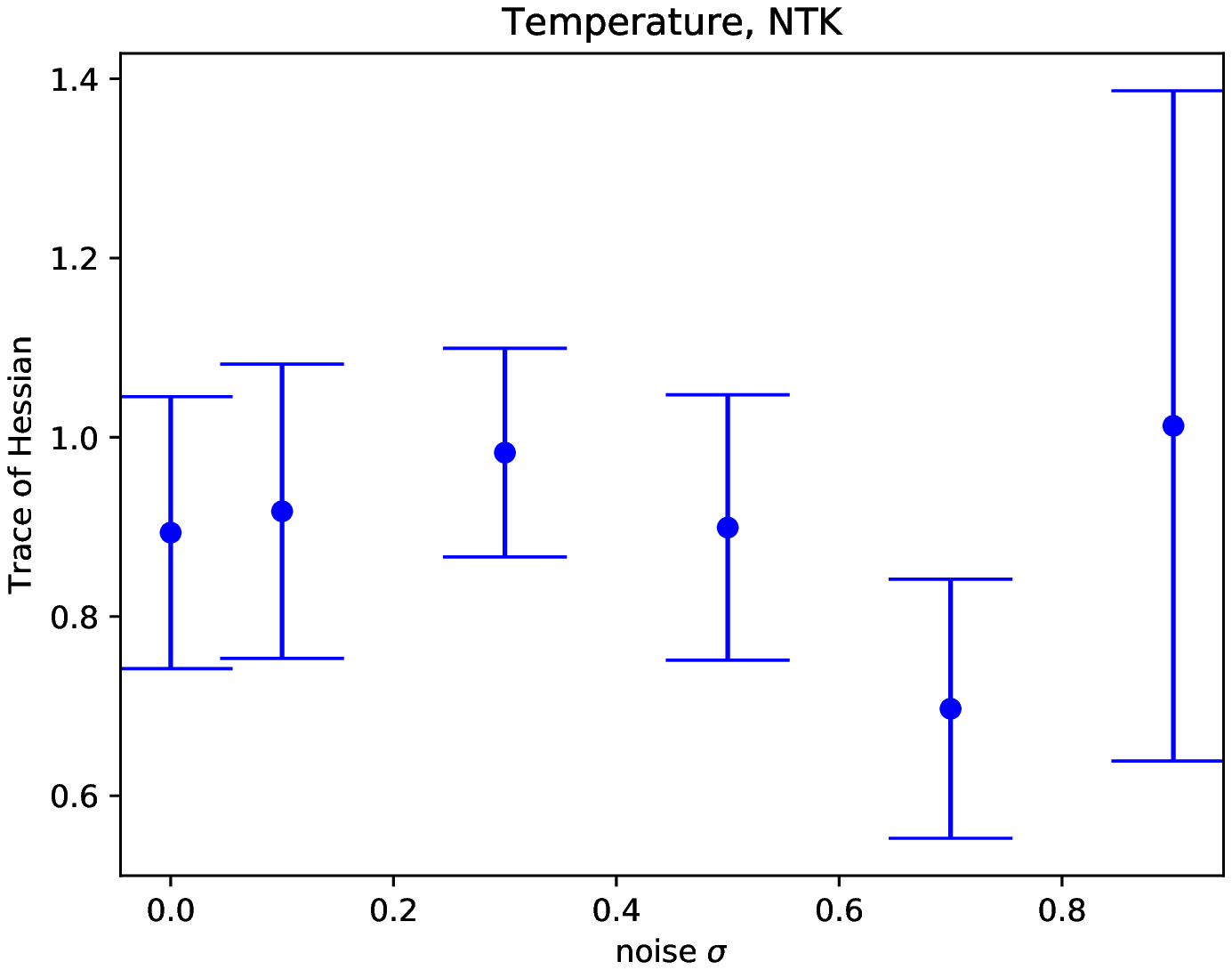}
        \caption{The effects of noisy training with 10000 iterations, with Gaussian noise with standard deviation $\sigma$ added in function space, with NTK scaling on the trace of the first- and last-layer weight Hessians on the sine (L) and temperature (R) data. The trace of the Hessian decreases for more noise, meaning that in this non-lazy regime noise regularizes the output function.}\label{fig3a1}
\end{figure}

In Figures \ref{fig5}-\ref{fig51} we show the results of SGD averaged over 20 runs for the regular scaling (here, He initialization). We modify the variance of the noise using the batch size and again consider 10000 training iterations. We observe that for neural networks in this non-lazy training regime, i.e. ones in which we do not scale the weights with the NTK scaling and train for a larger number of iterations so that the weights deviate significantly from the , the noise in the SGD has a significant regularization effect; in particular the noise from using a batch size of one results in a much lower test error. This is similar to the theoretical result obtained in Sections \ref{sec3}, in which it was shown that regularization effects due to noise seem to arise mostly in the regular training regime when the network function is \emph{non-linear} in the weights as well. The variance of the train error is higher when more noise is injected.  

\begin{figure}[h!]
  \centering
  \includegraphics[width=0.35\textwidth]{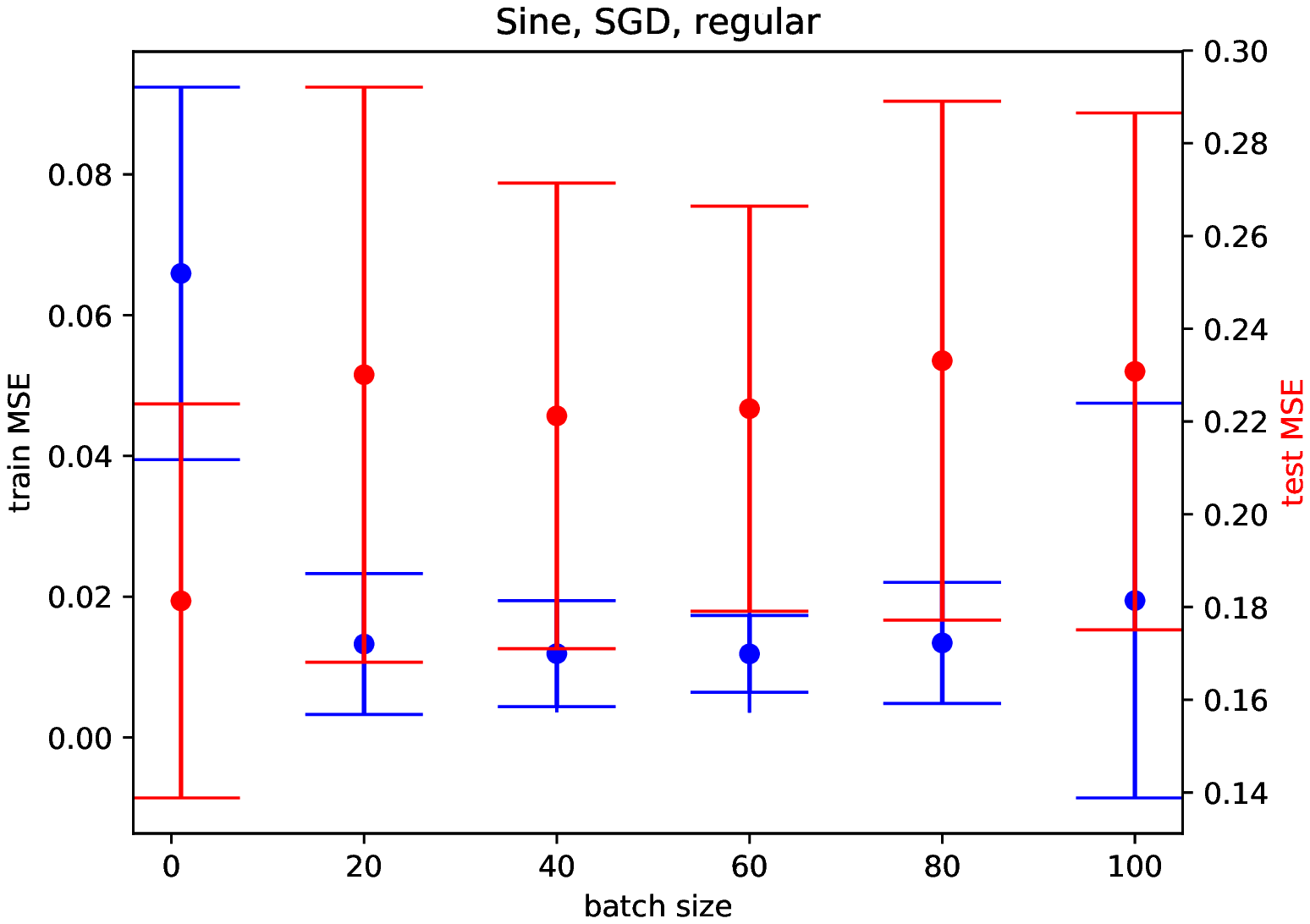}
 \includegraphics[width=0.35\textwidth]{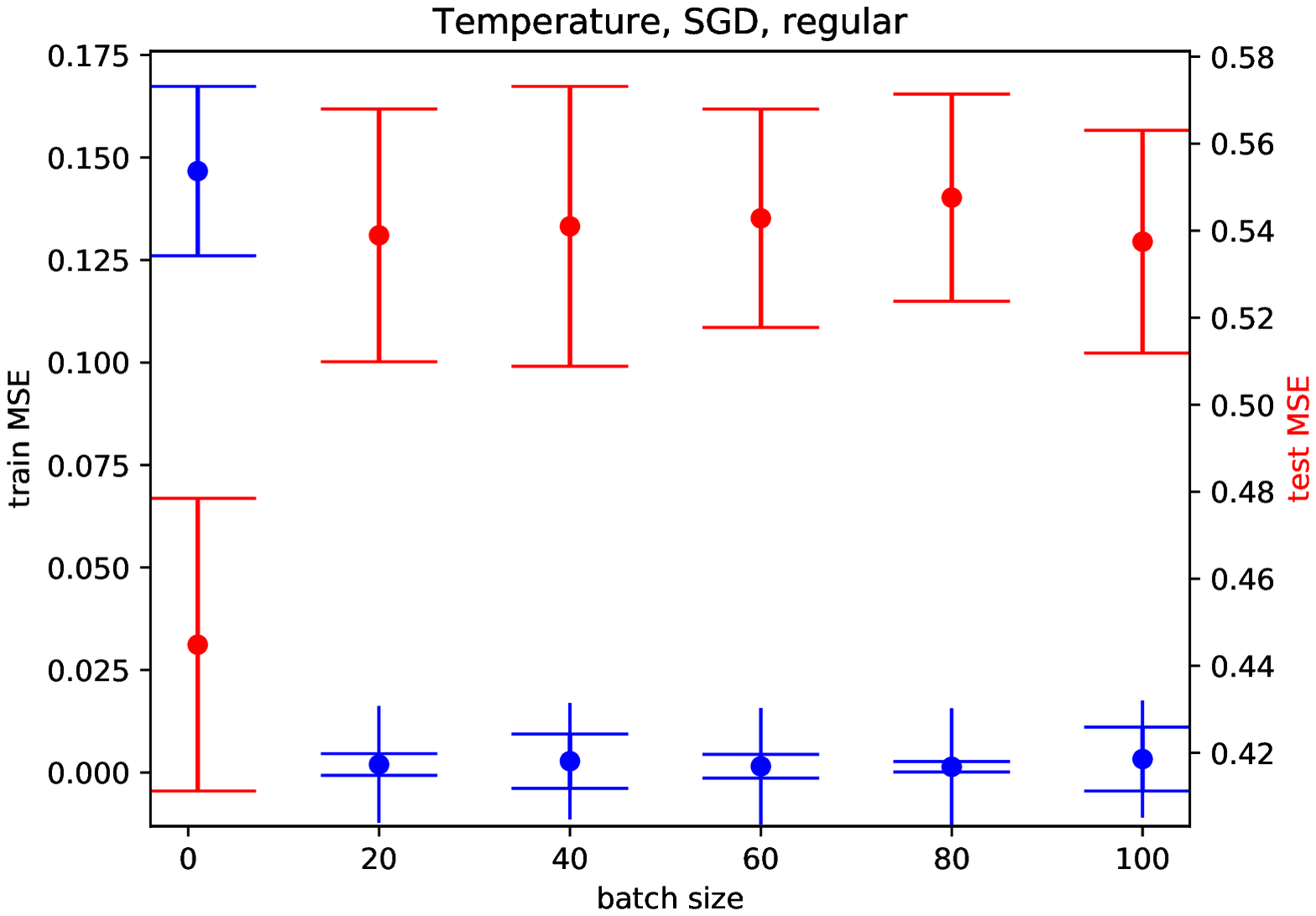}
        \caption{The effects of stochastic gradient descent with 10000 iterations in the non-lazy regime for the sine (L) and temperature (R) dataset on the train and test performance. In the non-lazy training regime the noise improves the test performance. In particular for a batch size of one the generalization error is significantly smaller.}\label{fig5}
\end{figure}

\begin{figure}[h!]
  \centering
   \includegraphics[width=0.35\textwidth]{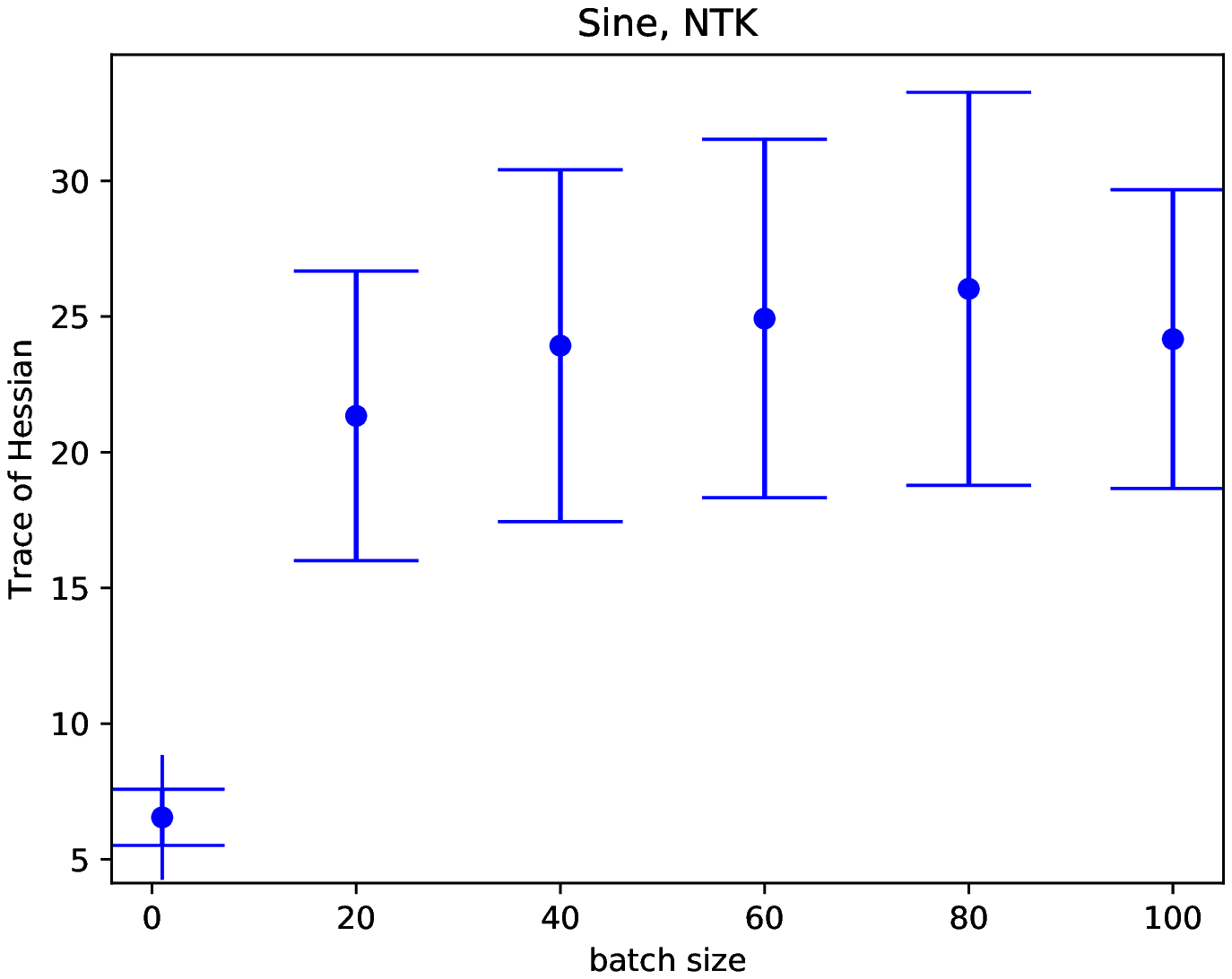}
\includegraphics[width=0.35\textwidth]{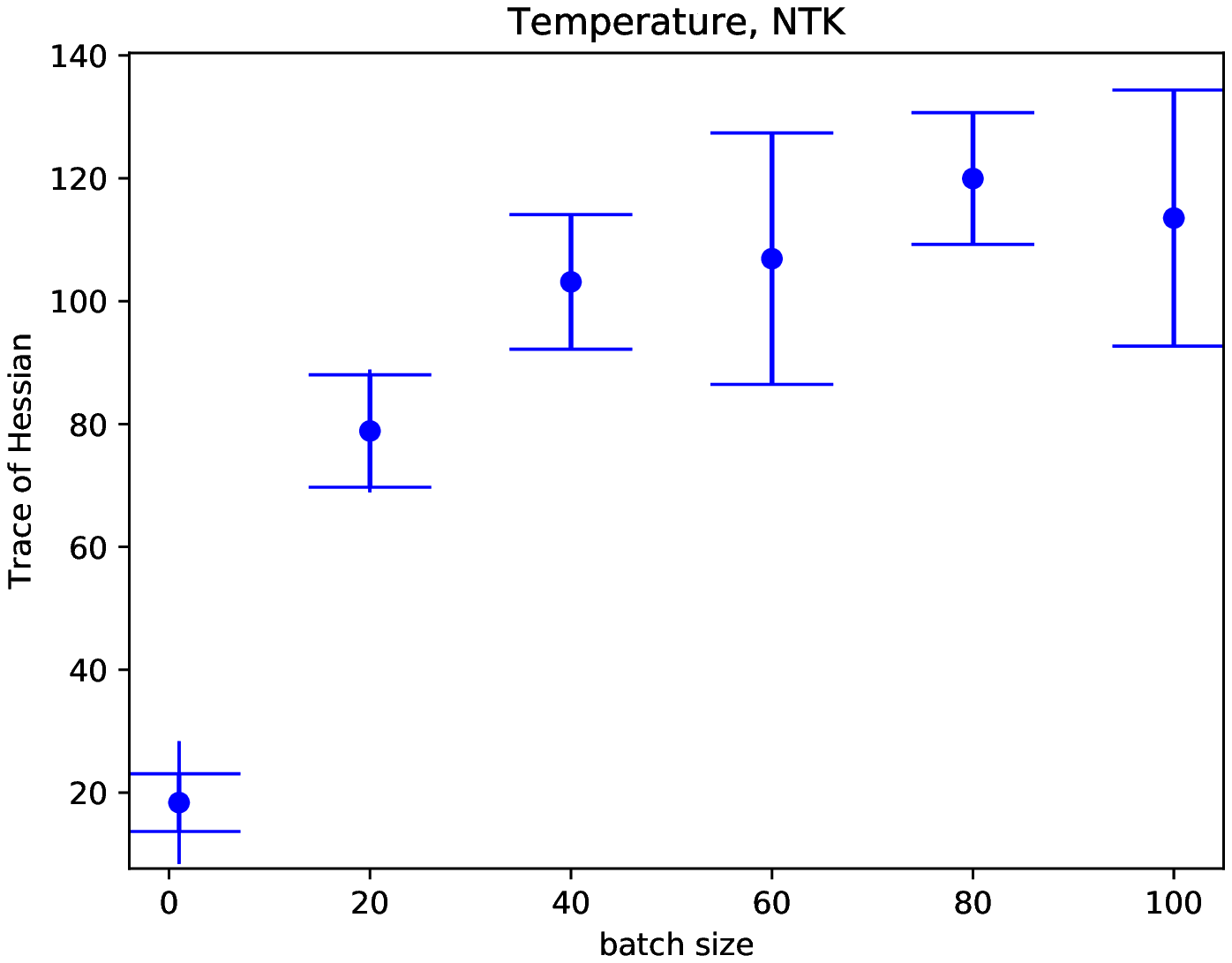}
        \caption{The effects of stochastic gradient descent in the non-lazy regime for the sine (L) and temperature (R) dataset on the trace of the first- and last-layer weight Hessians. As expected, more noise, and in particular the noise when using a batch size of one, results in a smaller weight Hessian.}\label{fig51}
\end{figure}

\section{Conclusion and discussion}
In this work we studied the effects of the training method on generalization capabilities of a neural network through analytic expressions of the network output during training. We studied the effect of the optimization hyperparameters in the lazy regime, i.e. when the network output can be approximated by a linear model. We observed that the addition of noise in the gradient descent updates in the lazy regime can keep the network from fully converging on the train data, however does not directly result in a regularizing or smoothing effect. 
In general, the first-order approximation of the neural network output might not be sufficient and one needs to take into account higher order terms. In particular this happens if the weights during training deviate significantly from their value at initialization which occurs for relatively narrow and/or deep networks and when the number of training iterations or the noise variance increases. Under stochastic training the expected value of the weights satisfies a Cauchy problem, which in general is unsolvable due to the state-dependency in the operator. We presented a novel methodology where a Taylor expansion of the network output was used to solve this Cauchy problem. This allowed to obtain analytic expressions for the weight and output evolution under stochastic training in a setting in which the model is no longer equivalent to its first order approximation. We showed that the higher order terms are affected by noise in the weight updates, so that unlike in the linear model noise has an explicit effect on the smoothness of the network output. 

The Taylor expansion method used in order to obtain analytic expressions for the network output can be seen as an extension of the linear model approximation and it allows to gain insight into the effects of higher-order terms. The method can be further extended to take into account high-dimensional weight vectors. These expressions can provide insight into the effects of network architecture, such as using a convolutional neural network or particular choices of activation function, on the output function evolution. Furthermore, the high-dimensional approximation can be used to study the effects of non-isotropic noise in the high-dimensional loss surface. Based on the observations in \cite{simsekli19}, noisy optimization converges to wider minima is obtained when using L\'evy-driven noise, while the Gaussian noise can actually stimulate the convergence to sharp and deeper minima. It is therefore of significant interest to obtain insight into the function evolution when optimization is driven by a L\'evy process. The methodology presented in this work can be adapted to include a L\'evy-driven jump process (see e.g. \cite{lorig15cf} where the evolution of the L\'evy-SDE is studied in Fourier space).

While understanding the effects of network hyperparameters on the output evolution and generalization capabilities remains a challenging task due to the interplay of so many aspects of the network (e.g. architecture, optimization algorithm and data), we hope that the Taylor expansion method and resulting insights presented in this work will contribute to novel insights and the development of robust and stable neural network models. 

\bibliographystyle{siam}
\bibliography{biblio}

\end{document}